\documentclass[11pt]{article}

\usepackage[utf8]{inputenc}
\usepackage[T1]{fontenc}
\usepackage{lmodern}
\usepackage[letterpaper,margin=1in]{geometry}

\usepackage[round]{natbib}

\usepackage{amsfonts}
\usepackage{amsthm}
\usepackage{amsmath}
\usepackage{mathtools}
\usepackage{xcolor}
\usepackage{multirow}
\usepackage{makecell}
\usepackage{graphicx}
\usepackage{booktabs}
\usepackage{algorithm}
\usepackage{algorithmic}
\usepackage{subfigure, graphicx, transparent}
\usepackage{url}
\usepackage{xurl}
\usepackage{comment}
\usepackage{placeins}
\usepackage{hyperref}
\usepackage{etoc}
\definecolor{tocblue}{HTML}{1A2F8A}

\usepackage{enumitem}
\setlist{nosep}

\usepackage[capitalise,nameinlink]{cleveref}
\crefname{theorem}{Theorem}{Theorems}
\crefname{proposition}{Proposition}{Propositions}
\crefname{corollary}{Corollary}{Corollaries}
\crefname{lemma}{Lemma}{Lemmas}
\crefname{definition}{Definition}{Definitions}
\crefname{assumption}{Assumption}{Assumptions}
\crefname{remark}{Remark}{Remarks}

\theoremstyle{plain}
\newtheorem{remark}{Remark}
\usepackage{bbm}
\usepackage{dsfont}  %

\def \ds{\displaystyle}
\def \1{{\mathbbm{1}}}

\def \E{{\mathbb E}}
\def \P{{\mathbb P}}

\DeclareMathOperator{\sign}{sign}

\newcommand{\relint}{\mathrm{relint}}

\newtheorem{theorem}{Theorem}
\newtheorem{proposition}[theorem]{Proposition}

\newtheorem{corollary}[theorem]{Corollary}
\theoremstyle{definition}

\theoremstyle{remark}

\newenvironment{restated}[1]{\par\medskip\noindent\textbf{\Cref{#1} (restated).}\itshape\ }{\par\medskip}

\newcommand{\aistatstitle}[1]{\begin{center}{\LARGE\bfseries #1\par}\end{center}\vspace{1em}}

\begin{document}

\etocdepthtag.toc{mtchapter}  

\title{V-ECE: Estimating General Expected Calibration Errors}

\author{
Eugène Berta$^{*,1,2}$ \and
Sacha Braun$^{*,1,2}$ \and
Francis Bach$^{1,2}$ \and
Michael I. Jordan$^{1,3}$ \and
David Holzmüller$^{1}$ \\[0.6em]
{\normalsize $^1$Inria \quad
$^2$École Normale Supérieure, PSL Research University \quad
$^3$UC Berkeley} \\[0.3em]
{\small $^*$Equal contribution, alphabetical order.}
}

\date{}

\maketitle

\begin{abstract}
In probabilistic classification, calibration error (CE) measures the average divergence of predicted probabilities $f(X)$ from $\P(Y|f(X))$, the true class distribution for that predicted probability.
While being a useful diagnostic tool, it is hard to estimate: popular binning-based estimators are often inconsistent and scale poorly beyond two classes.
Recent work rewrites the CE as the excess risk of a model compared to the best recalibration of its own predictions, measured with a proper loss.
However, this only works for Bregman-divergence-based calibration errors like the squared error, excluding the more popular $L_1$-distance-based CE.
We show that using prediction-dependent proper scores can alleviate this restriction, allowing us to estimate CEs with general convex divergences, including $L_p$ distances with closed-form losses in the binary and multiclass settings.
To estimate the excess risk, we introduce a more accurate recalibrator that fits a residual to temperature scaling with gradient boosting.
The resulting variational estimator, V-ECE, needs no bins or clusters and lower-bounds the true calibration error in expectation.
On a benchmark of semi-synthetic tasks built from real classifiers, with known true CE, V-ECE is among the most accurate binary estimators for every calibration error and significantly outperforms all multiclass estimators.
Our results are accompanied by additional theory on $L_p$ CE, estimator bias, and over- or under-confidence estimation.
\end{abstract}

\section{Introduction}

\begin{figure}[t]
    \centering
    \includegraphics[width=0.6\columnwidth]{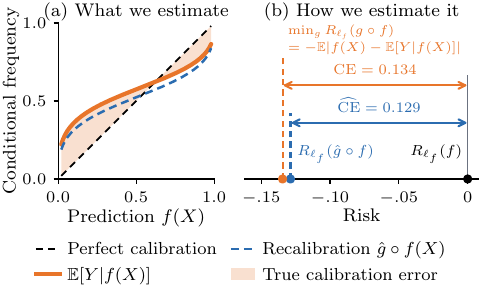}
     
    \caption{\textbf{The $L_1$ calibration error and how we estimate it}, for an over-confident binary model.
    (a) The $L_1$ CE is the orange area between $\mathbb{E}[Y|f(X)]$ and the diagonal $f(X)$, weighted by the distribution of $f(X)$.
    (b) Re-expressing the $L_1$ CE as a risk difference: writing $R_{\ell_f}(g \circ f) = \mathbb{E}[\ell_{f(X)}(g \circ f(X), Y)]$ for the risk under the prediction-dependent score of \Cref{cor:lp}, the largest drop achievable over all recalibration functions is exactly the $L_1$ CE. Any $\hat g$ therefore gives a lower bound.
    }
    \label{fig:intro}
\end{figure}

We study the problem of calibration in classification.
Let $\mathcal{X}$ be an arbitrary sample space and $\mathcal{Y}$ the space of one-hot-encoded categorical outcomes $\mathcal{Y} = \{ y \in \{ 0, 1 \}^k \mid \sum_{i=1}^k y_i = 1 \}$.
Consider $(X, Y) \in \mathcal{X} \times \mathcal{Y}$ such that $(X, Y) \sim \mathbb{P}$ for some unknown probability distribution $\mathbb{P}$.
The goal of classification is to predict~$Y$ using information from the feature $X$.

In modern machine learning, we usually tackle this class of problems using methods that make continuous predictions in the probability simplex $\Delta_k = \{ u \in [0, 1]^k \mid \sum_{i=1}^k u_i = 1 \}$.
The model $f:\mathcal{X} \rightarrow \Delta_k$ not only predicts the most likely class but produces a probability vector assigning a score in $[0,1]$ to each class, evaluating the likelihood of observing the associated outcome.
Such a classifier is called \emph{calibrated} when
$$
\mathbb{E}[Y \mid f(X)] = f(X) \quad \text{$\mathbb{P}$-almost surely.}
$$
Since the outcome vector $Y$ is one-hot encoded, the conditional expectation $C \coloneqq \mathbb{E}[Y | f(X)]$ is a probability vector in $\Delta_k$ with $C_i = \mathbb{P}(Y \! = \! e_i | f(X))$.
Calibration therefore means that the scores produced by the model align with ``real world'' probabilities, which makes its predictions trustworthy for the end user.

Unfortunately, it is often observed that machine learning classifiers are not calibrated ``out of the box'' and tend to produce unreliable predictions \citep{zadrozny2002transforming, guo2017calibration}.
This lack of calibration is quantified via a calibration error,
\begin{equation} \label{eq:CalibrationError}
    \mathrm{CE}_{f,d} = \mathbb{E}[d(f(X), C)] \, ,
\end{equation}
where $d$ is a divergence on $\Delta_k$.
The divergences that dominate practical usages are $L_p$ distances.
In the binary setting $\mathcal{Y} = \{ 0,1 \}$, $f:\mathcal{X} \rightarrow [0,1]$, the $L_1$ calibration error
\[
    \mathrm{CE}_{f,|\cdot|} = \mathbb{E}[| f(X) - C |]
\]
is widely used and is directly interpretable as the average probability mass by which the model is wrong.
In the multiclass case the natural extensions are again~$L_p$: the $L_1$ calibration error $\mathbb{E}[\| f(X) \! - \! C \|_1] = \sum_{i=1}^k \mathbb{E}[ | f(X)_i - C_i |]$ aggregates the binary error of each class against the rest, and the Euclidean calibration error $\mathbb{E}[\| f(X) \! - \! C \|_2]$ measures the average displacement of the prediction inside the simplex.

Calibration error is difficult to estimate because it involves the conditional expectation $C \! = \! \mathbb{E}[Y | f(X)]$, which is never observed and has to be estimated from data.
In the binary case this is classically done by binning the $[0,1]$ interval and computing the gap between predictions and average outcome in each bin, which gives the expected calibration error estimator \citep[ECE,][]{naeini2015obtaining}.
ECE is biased and inconsistent \citep{kumar2019verified, roelofs2022mitigating}, tends to overestimate calibration error, and requires choosing the number of bins.
Binning the simplex also suffers from the curse of dimensionality when $k > 2$, so it is common to fall back on the top-class calibration error in a one-versus-rest fashion \citep{guo2017calibration}.

A recent line of work avoids binning by using a variational identity.
When the divergence $d = d_\ell$ is induced by a proper loss $\ell$ \citep{gruber2022better}, \citet{berta2025rethinking} show that the calibration error of $f$ is the excess risk of $f$ over its best recalibration,
\begin{align}
\label{eq:CE_d:decomposition}
\mathrm{CE}_{f,d_\ell} = \mathbb{E}[\ell(f(X), Y)] - \min_{g\in\mathcal{H}} \mathbb{E}[\ell(g \circ f(X), Y)] \, ,
\end{align}
where $\mathcal{H}$ is the set of measurable maps from $\Delta_k$ to $\Delta_k$.
Estimating a calibration error becomes a matter of fitting a recalibration function $\hat g$ of $f(X)$ onto $Y$, a supervised learning problem for which many methods are available.
The identity holds unchanged in the multiclass setting, and when $\hat g$ is fitted by cross-validation the resulting estimate is a lower bound on the true calibration error in expectation (\cref{sec:lower_bound}).

The limitation of this approach is the set of calibration errors it can estimate.
A proper loss is characterized by a concave entropy $H$ and induces the Bregman divergence generated by $-H$ \citep{gneiting2007strictly}: the squared Euclidean and KL calibration errors are reachable through the Brier score and the log-loss, but the $L_1$ and $L_2$ calibration errors are not.
We show that this restriction comes from using a single proper score for all predictions, while the identity \eqref{eq:CE_d:decomposition} is applied pointwise in $f(X)$.
Following \citet{braun2025conditional}, we let the entropy depend on the prediction and write $\ell_{f(X)}$ for the score used at that prediction.
Each $\ell_{f(X)}$ remains proper, so the recalibration argument and the lower bound are unaffected, but the induced divergence need not be a Bregman divergence: taking $H_u(z) = - \| z-u \|_p$ makes the variational estimator return the $L_p$ calibration error as we show in \Cref{sec:method}. This is illustrated in \Cref{fig:intro} for $p=1$. 

\paragraph{Contributions.} In this paper we introduce
\begin{itemize}[topsep=2pt,leftmargin=*]
    \item prediction-dependent proper scores that extend the variational estimator of \citet{berta2025rethinking} from Bregman divergences to any convex divergence (see \cref{thm:general}) on the simplex, with closed-form $L_p$ losses for every $p \ge 1$ and a proof that no fixed proper loss can express them;
    \item a bias analysis: the bias depends only on the direction of $\hat g \circ f - f$, and for $p \! = \! 1$ reduces to a misclassification error weighted by the calibration gap, which explains why $L_1$ calibration errors are easier to estimate than proper ones in our framework;
    \item an exact decomposition of the $L_1$ calibration error into over- and under-confidence, each component being estimated by the same procedure;
    \item V-ECE, the estimator we propose based on these results. It uses a warm-started gradient-boosting recalibrator, which better recovers calibration error than standard recalibration methods. We compare it with existing estimators on a semi-synthetic benchmark built from CalArena \citep{berta2026calarena}, where it is among the most accurate of 17 binary estimators and the most accurate of 5 multiclass estimators. Code is available at \href{https://github.com/probkit/V-ECE}{https://github.com/probkit/V-ECE}.
\end{itemize}

Throughout, $\mathrm{CE}_{f,d}$ denotes a calibration error of $f$ and, when it is measured with a loss $\ell$ and a recalibration function $g$, we write $\mathrm{CE}_{f,d_\ell}(g) \coloneqq \mathbb{E}[\ell(f(X),Y)] - \mathbb{E}[\ell(g \circ f(X), Y)]$, and $\widehat{\mathrm{CE}}_{f,d_\ell}(g)$ for empirical versions.

\begin{table}[t]
    \centering
    \begin{tabular}{lcc}
    \toprule
    \textbf{Metric} & \textbf{Loss employed} & \textbf{CE estimated} \\
    \midrule
    \textbf{Cross-entropy} & $- \sum_{i} Y_i \log(z_i)$ & $\mathbb{E}[\mathrm{KL}(C \| f)]$ \vspace{0.3em} \\
    \textbf{Brier score} & $\| z - Y \|_2^2$ & $\mathbb{E}[\| f - C \|_2^2]$ \vspace{0.3em} \\
    \midrule
    \textbf{Ours ($\mathbf{L_1}$)} & $\mathds{1}_{z \neq f} \langle \sign(z \! - \! f) , f \! - \! Y \rangle$ & $\mathbb{E}[\| f - C \|_1]$ \vspace{0.3em} \\
    \textbf{Ours ($\mathbf{L_2}$)} & $\mathds{1}_{z \neq f} \langle \frac{z - f}{\| z - f \|_2} , f \! - \! Y \rangle$ & $\mathbb{E}[\| f \! - \! C \|_2]$ \\
    \textbf{Ours (cvx $D_{f}$)} & $-D_{f}(z) \! - \! \langle Y \! - \! z, \nabla D_{f}(z) \rangle$ & $\mathbb{E}[D_{f}(C)]$ \vspace{0.3em} \\
    \bottomrule
    \end{tabular}
    \caption{\textbf{Losses and the calibration error they estimate}, writing $f$ for $f(X)$.
     The estimate is the risk of $f$ minus the risk of $\hat g \circ f$ under the loss of the second column.
     The first two rows are fixed proper losses and give Bregman divergences; the next two are prediction-dependent and give $L_p$ distances, which no fixed proper loss can produce (\Cref{prop:no:go}); and the last row works for any convex $D_f$ that satisfies the conditions of \Cref{thm:general}.}
    
    \label{table:losses}
\end{table}

\section{Related work}
\label{app:related}

Calibration originates in the forecasting literature \citep{murphy1973vector, degroot1983comparison, foster1998asymptotic}, where the decomposition of any proper score into calibration and refinement was established \citep{brocker2009reliability}, the calibration and sharpness view was formalized \citep{gneiting2007probabilistic}, and reliability diagrams are the standard diagnostic \citep{brocker2007increasing, dimitriadis2021stable}.
Interest from the machine learning literature was renewed by the observation that deep networks are systematically over-confident \citep{guo2017calibration}, which later architectures only partly correct \citep{minderer2021revisiting}.
Several notions of multiclass calibration coexist, from top-label \citep{gupta2022top} to canonical \citep{vaicenavicius2019evaluating, kull2019beyond, gruber2022better}, alongside stronger requirements such as multicalibration \citep{hebert2018multicalibration} and calibration relative to downstream decisions \citep{zhao2021calibrating}.
A separate and complementary line, initiated by \citet{blasiok2023unifying}, measures the distance from $f$ to the set of calibrated predictors rather than a divergence between $f(X)$ and $C$, and shows that several commonly used calibration measures are not polynomially related to that distance.

The expected calibration error \citep{naeini2015obtaining, guo2017calibration} bins the range of $f$ and compares the average prediction to the average outcome in each bin.
It is biased upwards, inconsistent \citep{vaicenavicius2019evaluating, kumar2019verified, roelofs2022mitigating}, and dependent on the binning scheme; adaptive schemes \citep{nixon2019measuring}, debiasing \citep{kumar2019verified, roelofs2022mitigating} and distribution-free variants \citep{gupta2021distribution} address parts of this.
\citet{blasiok2024smooth} instead replace binning by kernel smoothing with an automatically chosen bandwidth, and obtain both a reliability diagram and a scalar measure consistent with the distance from calibration; like ours their estimator removes the discretization, but it reports a smoothed quantity rather than a one-sided bound on $\mathrm{CE}_{f,\|\cdot\|_1}$ itself.
\citet{lee2023tcal} give minimax rates for \emph{testing} calibration under smoothness assumptions on $C$, the testing counterpart of the estimation problem studied here.

Other kernel estimators smooth rather than bin: \citet{popordanoska2022consistent} give a consistent estimator of $L_p$ canonical calibration errors and \citet{popordanoska2024consistent} treat proper calibration errors, in both cases at the price of a bandwidth and without a one-sided guarantee \citep[see also][]{zhang2020mix}.
A third family defines calibration errors as a supremum over test functions, which test the residual $Y - f(X)$ against functions of the prediction, $\mathrm{CE}_{\mathcal{W}}(f) = \sup_{w \in \mathcal{W}} \mathbb{E}\big[ \langle w(f(X)), Y - f(X) \rangle \big] \, ,
$ for a class $\mathcal{W}$ of maps $w : \Delta_k \to \mathbb{R}^k$.
This formulation goes back to the weak calibration of \citet{kakade2008deterministic}, and the choice of $\mathcal{W}$ defines the metric: the maximum mean calibration error of \citet{kumar2018trainable} and the kernel calibration error of \citet{widmann2019calibration, widmann2021calibration} take the unit ball of a reproducing kernel Hilbert space, and the smooth calibration error of \citet{blasiok2023unifying} takes $1$-Lipschitz functions bounded by $1$.
For these classes, the supremum has a closed form or a tractable estimator, but the resulting quantity is a new metric, not the $L_p$ calibration error.
The $L_p$ calibration error is recovered with the largest class: for $p \ge 1$ with conjugate exponent $q$, H\"older's inequality gives $\|C - f(X)\|_p = \sup_{\|w\|_q \le 1} \langle w, C - f(X) \rangle$ at every prediction, and conditioning on $f(X)$ turns $C - f(X)$ into $Y - f(X)$, so that $\mathrm{CE}_{f,\|\cdot\|_p} = \sup_{w \,:\, \|w(\cdot)\|_q \le 1} \mathbb{E}\big[ \langle w(f(X)), Y - f(X) \rangle \big] \,$, the supremum running over all measurable maps $w$ with $\|w(u)\|_q \le 1$ for every $u$.

Existing estimators of the $L_p$ calibration error do not use this form: they estimate $C$ by a function $\hat g$ and return the plug-in $\frac{1}{n} \sum_{i=1}^n \|\hat g(f(X_i)) - f(X_i)\|_p$, with no one-sided guarantee.
\citet{kangsepp2025usefulness} make this template explicit with their \emph{fit-on-test} view: any post-hoc calibration method fitted on the evaluation data yields such a plug-in estimator, binned ECE \citep{naeini2015obtaining, guo2017calibration} being the special case of a piecewise-constant~$\hat g$ fitted with the Brier score, and kernel regression \citep{popordanoska2022consistent} fits the same template.
We share their premise that recalibration methods can serve as calibration estimators, but change what is computed from $\hat g$: an excess risk on held-out points rather than a plug-in distance on the points used to fit $\hat g$.
This removes the test-time overfitting that \citet{kangsepp2025usefulness} identify as the main limitation of fit-on-test, which biases the estimate upwards, and replaces it with a downward bias controlled by \Cref{thm:alignment}; cross-validation then serves to evaluate the estimate, not only to tune $\hat g$ as in their cross-validated choice of the number of bins.

Proper losses relate to calibration in the other direction as well: \citet{blasiok2023proper} ask when a predictor that nearly minimizes a proper loss is calibrated, and answer in terms of how much the loss can still be decreased by post-processing, which is the quantity that the variational identity equates with the calibration error.
We use it as a measurement device rather than a training criterion, and widen the set of divergences it can express.

The methods we use to estimate $g^\star$ are standard recalibration methods: Platt scaling \citep{platt1999probabilistic}, isotonic regression \citep{zadrozny2002transforming}, Beta calibration \citep{kull2017beta}, temperature and matrix scaling \citep{guo2017calibration, kull2019beyond, berta2025structured}.
The variational estimator we extend is due to \citet{berta2025rethinking} (binary and multiclass), and \citet{braun2025conditional} introduced proper scores that are allowed to depend on external fixed values (binary only).
We extend these to proper scores that depend on the prediction itself, and to the multiclass setting.

Compared with these approaches, our construction unifies the two views of the calibration error.
It shows that the supremum over test functions defining the $L_p$ calibration error is the excess risk of a prediction-dependent proper score, which places $L_p$ calibration errors in the same excess-risk picture as proper calibration errors.
Existing estimators then appear as instances of a single identity: the binned ECE uses a test function that is constant on each bin and chosen on the evaluation points themselves, test-function estimators optimize it over a restricted class, and the variational estimators of \citet{berta2025rethinking} and the CORP decomposition \citep{dimitriadis2021stable} use a fixed proper loss, hence a Bregman divergence.

This framework provides what was missing for the $L_p$ calibration error: a lower bound in expectation for any recalibration function fitted on held-out data (\Cref{prop:lowerbound}), whereas test functions chosen on the evaluation points overfit the labels and bias the estimate upwards; a bias analysis showing that only the direction of $\hat g - f$ matters (\Cref{thm:alignment}); a geometrical interpretation of the calibration error as the gap under a supporting line of a prediction-dependent entropy with a kink at the prediction (\Cref{app:intuition}); and an extension to convex divergences that are not norms, such as the one-sided divergences that separate over- from under-confidence (\Cref{prop:overunder}).

\section{Proper calibration errors}
\label{sec:background}

A loss $\ell : \Delta_k \times \mathcal{Y} \to \mathbb{R}$ is \emph{proper} if reporting the true conditional distribution is optimal, that is, $\mathbb{E}_{Y \sim v}[\ell(u, Y)] \ge \mathbb{E}_{Y \sim v}[\ell(v, Y)]$ for all $u, v \in \Delta_k$. $\ell$ is \emph{strictly proper} if equality holds iff $u \! = \! v$.
A proper loss is fully characterized by a concave function $H:\Delta_k \rightarrow \mathbb{R}$ via
\begin{equation} \label{eq:proper:from:H}
    \ell(u,v) = H(u) + \langle \delta H(u), v-u \rangle \, ,
\end{equation}
where $\delta H(u)$ is a supergradient of $H$ at~$u$ and $\ell(u,v) = \mathbb{E}_{Y\sim v}[\ell(u,Y)]$ \citep{gneiting2007strictly}.
The associated \emph{entropy} is $e_\ell(u) = H(u)$ and the induced \emph{divergence} is
\begin{equation} \label{eq:bregman}
    d_\ell(u,v) = H(u) - H(v) + \langle \delta H(u), v - u \rangle \, ,
\end{equation}
the Bregman divergence generated by the convex function $-H$.

These divergences appear in a known decomposition of the risk of a classifier.
\citet{brocker2009reliability} showed that for any proper loss $\ell$,
$$
    \mathbb{E}[\ell(f(X), Y)] = \mathbb{E}[d_\ell(f(X), C)] + \mathbb{E}[e_\ell(C)] \, .
$$
The first term is the calibration error of $f$ measured with $d_\ell$, and such calibration errors are called \emph{proper calibration errors} \citep{gruber2022better}.

\citet{berta2025rethinking} show that these proper calibration errors can be written as
\begin{equation} \label{eq:VariationalCalibration}
    \mathrm{CE}_{f,d_\ell} = \mathbb{E}[\ell(f(X), Y)] - \min_{g\in\mathcal{H}} \mathbb{E}[\ell(g \circ f(X), Y)] \, ,
\end{equation}
where the minimum is over the set $\mathcal{H}$ of all measurable functions from $\Delta_k$ to $\Delta_k$.
We write $g^\star$ for the optimal recalibration function, defined by $g^\star(f(X)) = \mathbb{E}[Y | f(X)]$. For any strictly proper loss it minimizes $\mathbb{E}[\ell(g\circ f(X),Y)]$ over $\mathcal{H}$.
Equivalently, $g^\star$ is the Bayes predictor of $Y$ from covariates $f(X)$, and the calibration error is the excess risk of $f(X)$ over it, a reading also noted by \citet{blasiok2023proper}.

Given an estimate $\hat{g}$ of $g^\star$, a proper calibration error is estimated by the risk of $f$ minus the risk of $\hat{g} \circ f$,
\begin{equation} \label{eq:CE_hat}
    \widehat{\mathrm{CE}}_{f,d_\ell}(\hat g) = \frac{1}{n} \sum_{i=1}^n \ell(f(X_i), Y_i) - \ell( \hat{g} \circ f(X_i), Y_i) \, .
\end{equation}
Since $g^\star$ is the conditional expectation of $Y$ given $f(X)$, an estimate~$\hat{g}$ can be obtained with any classification algorithm targeting $Y$ from the feature vector $f(X)$, so the procedure applies directly in the multiclass setting and requires no binning of the simplex.

\section{Estimating general and $L_p$ calibration errors}
\label{sec:method}

Nothing in the variational identity requires the loss used at $f(X)=u$ to be the loss used at $f(X)=v$.
Letting the entropy function depend on the prediction keeps each $\ell_{f(X)}$ proper, so $g^\star$ is still its population minimizer, while the family $\{H_u\}_{u\in\Delta_k}$ can be non-smooth at $u$ in $H_u$ for every $u$, which a single entropy function cannot do. Rewriting $d(f(X), C)$ as $D_{f(X)}(C)$, the following theorem introduces new prediction-dependent proper scores whose divergences are exactly the convex divergences on the simplex.

\subsection{The general construction}

\begin{theorem}[Convex divergences] \label{thm:general}
For each $u \in \Delta_k$, let $D_{u}: \Delta_k \to \mathbb{R}_{\geq 0}$ be a convex function minimized at $u$ with $D_{u}(u) = 0$, and let $\nabla D_{u}$ be a subgradient selection with $\nabla D_{u}(u) = 0$.
Define, for $z \in \Delta_k$ and one-hot $y \in \mathcal{Y}$,
\begin{equation} \label{eq:general:loss}
    \ell_{u}(z, y) \coloneqq - D_{u}(z) + \langle - \nabla D_{u}(z),  y - z \rangle \, .
\end{equation}
Then $\ell_u$ is a proper loss for every $u$, $\ell_u(u, y) = 0$, and
\begin{equation} \label{eq:general:variational}
    \mathbb{E}[D_{f(X)}(C)] = - \min_{g\in\mathcal{H}} \mathbb{E}[\ell_{f(X)}(g\circ f(X), Y)] \, ,
\end{equation}
with $\mathcal{H}$ the set of measurable functions from $\Delta_k$ to $\Delta_k$ and $C = \mathbb{E}[Y|f(X)]$.
\end{theorem}

The proof is deferred to \Cref{app:proofs}. The only requirement on $D_u$ is that it be convex and vanish at $u$, so the reachable class is all convex (in the second argument) divergences on the simplex, which strictly contains the Bregman divergences of \Cref{sec:background} for which the associated entropy $H_u$ is independent of $u$.

\subsection{$L_p$ calibration errors}

Taking $D_{u}(z) = \|z - u\|_p$ gives closed-form losses for every $p \ge 1$.

\begin{corollary}[$L_p$ calibration errors] \label[corollary]{cor:lp}
For $p \geq 1$, $z \in \Delta_k$ and $y \in \mathcal{Y}$, define
\begin{equation} \label{eq:lp:loss}
    \ell_u(z, y) \coloneqq \mathds{1}_{z \neq u} \langle \nabla_z \| z - u \|_p , u - y \rangle ,
\end{equation}
where
$$
\nabla_z \| z - u \|_p = \mathrm{sign}(z-u) \odot \frac{| z - u |^{p-1}}{\| z - u \|_p^{p-1}} \, ,
$$
with element-wise power and $\mathrm{sign}$ in the numerator (for $p=1$, $\nabla_z \| z - u \|_p =\mathrm{sign}(z - u)$).
Then
$$
\mathrm{CE}_{f,\|\cdot\|_p} = - \mathbb{E}[\ell_{f(X)}(g^\star\circ f(X), Y)] \, .
$$
\end{corollary}

The proof is in \Cref{app:proofs}, and a geometrical interpretation for $p \! = \! 1$ is provided in  \Cref{app:intuition}.
\Cref{table:losses} lists the losses obtained for $p \! = \! 1$ and $p \! = \! 2$ next to two standard proper losses and the quantity each one estimates.

\subsection{Empirical estimator}\label{sec:lower_bound}
To estimate the CE of a classifier $f$, we fit recalibration functions $(\hat{g}_j)_{1 \le j \le m}$ with $m$-fold cross-validation, we compute, on each holdout fold $\mathcal{I}_j$,
\begin{equation} \label{eq:lp:estimator}
\widehat{\mathrm{CE}}_{f,\|\cdot\|_p}(\hat g_j) = - \frac{1}{|\mathcal{I}_j|} \sum_{i \in \mathcal{I}_j} \ell_{f(X_i)}(\hat{g}_j \circ f(X_i), Y_i) \; ,
\end{equation}
and average the $m$ estimates.
We detail the procedure in \Cref{app:algorithm}.

The quantity our estimator targets is a risk minus a \emph{minimum} over~$g$, so substituting any particular $\hat g$ gives a lower bound in expectation, and taking an empirical risk (finite samples) yields an unbiased estimate:
\begin{align*}
    \mathbb{E} [\widehat{\mathrm{CE}}_{f,D}(\hat g)] = \mathrm{CE}_{f,D}(\hat g) \le \mathrm{CE}_{f,D} \text{ for every } \hat g~.
\end{align*}
Specifically, the first equality holds if $\hat g$ does not depend on the points used to evaluate it, since otherwise it could overfit to them. In our case, the equality holds since we use cross-validation and do not fit $\hat g_j$ on fold $\mathcal{I}_j$. Hence, every $\hat{g}$ yields a lower bound of the true CE in expectation, and better $\hat{g}$ will yield a tighter bound.
We provide a precise statement in \Cref{prop:lowerbound} in \Cref{app:proofs} and investigate the size of the gap in \Cref{sec:alignment}.

\begin{remark}\label{rem:unbiased}
One might hope for an unbiased estimator instead.
Even in the simplest possible setting we show this is out of reach (\Cref{prop:toy} in \Cref{app:conservative}).
\end{remark}

\subsection{Comparison with binning}
\label{sec:vs:binning}

For $p \! = \! 1$ and $k \! = \! 2$, our estimator and the expected calibration error (ECE) estimator of \citet{naeini2015obtaining} target the same quantity, and ECE can itself be read as fitting $\hat g$ by a histogram on the range of~$f$. They differ in two independent ways.
First, in how $\hat g$ is fitted: the histogram makes the estimate depend on the number and placement of bins, and partitioning the simplex scales poorly with the number of classes; we fit $\hat g$ with any classification algorithm taking $f(X) \in \Delta_k$ as input.
Second, in what is computed from $\hat g$: ECE is a plug-in distance $|f(X) - \hat g(f(X))|$ on the data used to fit the histogram, which noise in $\hat g$ inflates, so it is positive even for a perfectly calibrated model \citep{kumar2019verified, roelofs2022mitigating}; our estimator is an excess risk evaluated on held-out data, which a poor $\hat g$ can only decrease.

\subsection{The bias of the estimator}
\label{sec:alignment}

In practice the estimator is computed with an imperfect $\hat g$, so what matters is the bias at an arbitrary recalibration function.
Here $\mathrm{CE}_{f,\|\cdot\|_p}(g) = -\mathbb{E}[\ell_{f(X)}(g\circ f(X), Y)]$, the first term being absent since it vanishes for the losses of \Cref{thm:general}.
The purpose of this section is to provide an analysis of the cost we pay by estimating $\mathrm{CE}_{f,\|\cdot\|_p}(g^\star)$ with $\mathrm{CE}_{f,\|\cdot\|_p}(\hat g)$.
For $a, b \ne 0$, define the alignment coefficient
$$
\rho_p(a, b) \coloneqq \frac{\big\langle \nabla_z \|z\|_p |_{z=a},\; b \big\rangle}{\|b\|_p} \in [-1, 1] \, ,
$$
where the range follows from H\"older's inequality; for $p \! = \! 2$ it is the cosine similarity between $a$ and $b$, and for $p \! = \! 1$ it is $\langle \sign(a), b\rangle / \|b\|_1$.

\begin{theorem}[Alignment] \label{thm:alignment}
Let $p \ge 1$. $\forall g \in \mathcal{H}$,
\begin{equation} \label{eq:alignment}
\begin{split}
    &\mathrm{CE}_{f,\|\cdot\|_p} - \mathrm{CE}_{f,\|\cdot\|_p}(g) \\
    &\qquad = \mathbb{E}\big[ \|C - f(X)\|_p ( 1 - \rho_p(a_g, b) ) \big] \ge 0 \, ,
\end{split}
\end{equation}
where $a_g \coloneqq g\circ f(X) - f(X)$ and $b \coloneqq C - f(X)$.
Consequently:
\begin{enumerate}[topsep=0pt,leftmargin=*,label=(\roman*)]
    \item the bias is invariant under $a_g \mapsto \lambda a_g$ for any $\lambda > 0$, so it depends on the direction of $g \circ f - f$ and not on its magnitude;
    \item for $p \! = \! 1$, the bias is zero as soon as $\sign(g_i \circ f(X) - f(X)_i) = \sign(C_i - f(X)_i)$ for every $i$ with $C_i \ne f(X)_i$, almost surely, and for $p \in (1,\infty)$ if and only if $C - f(X)$ is almost surely a nonnegative multiple of $g \circ f(X) - f(X)$.
\end{enumerate}
\end{theorem}

The proof is deferred to \Cref{app:proofs}. With a fixed proper loss $\ell$ the same computation in \Cref{app:proofs} gives instead
\begin{equation} \label{eq:proper:bias}
    \mathrm{CE}_{f,d_\ell} - \mathrm{CE}_{f,d_\ell}(g) = \mathbb{E}[ d_\ell( g\circ f(X), C ) ] \, .
\end{equation}
A strictly proper calibration error is thus estimated exactly only when $g$ recovers $C$, while an $L_p$ calibration error is exact on a whole cone of recalibration functions around $C$, and for $p=1$ on all sign-correct ones: estimating $\mathrm{CE}_{f,\|\cdot\|_1}$ is a classification problem, recovering the sign of $C_i-f(X)_i$, rather than a regression problem.

Specializing to $p=1$ makes this concrete.
Let $\hat g$ be independent of the evaluation sample and write $a \coloneqq \hat g \circ f(X) - f(X)$ and $b \coloneqq C - f(X)$. Then
\begin{align*}
\mathrm{CE}_{f,\|\cdot\|_1} - \mathrm{CE}_{f,\|\cdot\|_1}(\hat g) \le 2 \sum_{i=1}^{k} \mathbb{E}\big[ |b_i| \, \mathds{1}_{\sign(a_i) \neq \sign(b_i)} \big] \, .
\end{align*}
The bias is bounded by a misclassification error weighted by the local calibration gap: only the coordinates on which $\hat g$ picks the wrong side of $f(X)$ contribute, each in proportion to how miscalibrated the model is there.
The two sources of difficulty work against each other, since where the model is nearly calibrated a mistake is likely but costs almost nothing, and where it is badly miscalibrated the correct side is easier to identify.
This is what allows a crude recalibration function to estimate $\mathrm{CE}_{f,\|\cdot\|_1}$ accurately, and it is the mechanism behind the pattern we observe in \Cref{sec:experiments}.

\subsection{Over- and under-confidence}
\label{sec:overunder}

A calibration error does not say in which direction the model is wrong. In this subsection, we show that our variational estimator can also recover this information.
In the binary case, write $f(X) \in [0,1]$ for the predicted probability of class~$1$ and $s \coloneqq \sign(f(X) - 1/2)$ for the class the model favors.
The model is \emph{over-confident} at $f(X)$ when the true probability of the class it favors is lower than predicted, $s\,(f(X) - C) > 0$, and \emph{under-confident} when $s\,(C - f(X)) > 0$.
Writing $(a)_+ = \max\{a,0\}$, we measure each failure mode separately through
\begin{align*}
    \mathrm{CE}^{+}_{f,|\cdot|} &\coloneqq \mathbb{E}\big[ \big( s\,(f(X) - C) \big)_+ \big] \, , \\
    \mathrm{CE}^{-}_{f,|\cdot|} &\coloneqq \mathbb{E}\big[ \big( s\,(C - f(X)) \big)_+ \big] \, .
\end{align*}

\begin{proposition}[Over- and under-confidence] \label[proposition]{prop:overunder}
Assume $\mathbb{P}(f(X) = 1/2) = 0$. Then
\begin{enumerate}
    \item[(i)] $\mathrm{CE}_{f,|\cdot|} = \mathrm{CE}^{+}_{f,|\cdot|} + \mathrm{CE}^{-}_{f,|\cdot|}$;
    \item[(ii)] $\mathrm{CE}^{\pm}_{f,|\cdot|} = -\mathbb{E}[\ell^{\pm}_{f(X)}(g^\star \circ f(X), Y)]$, with
    $$
    \ell^{\pm}_{f(X)}(z, Y) = \pm\, s\, (Y - f(X))\, \mathds{1}_{\pm s (f(X) - z) > 0} \, .
    $$
\end{enumerate}
\end{proposition}

(i) holds because the two events partition the space up to $\{C = f(X)\}$.
For (ii), note that $\mathrm{CE}^{\pm}_{f,|\cdot|} = \mathbb{E}[D^{\pm}_{f(X)}(C)]$ for the one-sided distance $D^{\pm}_{f(X)}(z) = ( \pm\, s\,(f(X) - z) )_+$. This distance is convex and vanishes at $f(X)$, so \Cref{thm:general} applies and gives $\ell^{\pm}$. The full proof is in \Cref{app:proofs}.
Each component is estimated by the same cross-validated procedure and inherits the same lower bound.
Together, the three numbers turn a calibration error into a diagnosis: they say not only that a model is miscalibrated but whether it is over- or under-confident, and by how much in each direction.
In the multiclass setting, this applies to the top-class prediction, one-versus-rest.
\Cref{thm:general} also covers other convex divergences, such as weighted or truncated ones (\Cref{app:general:distances}).

\section{Experiments}
\label{sec:experiments}

We evaluate the variational estimator on the predictions of real classifiers, in two steps.
\Cref{sec:exp:recalibrators} asks which recalibration function $\hat g$ the estimator should use.
This leads to a gradient-boosting recalibrator, WS-CatBoost.
\Cref{sec:exp:semisynthetic} compares our estimator (V-ECE) with existing calibration error estimators on semi-synthetic tasks whose true calibration error is known.
Implementation details are in \Cref{app:exp:details} and additional results in \Cref{app:exp:results}. \Cref{app:synthetic} illustrates the over- or under-confidence decomposition.
The code for the experiments is available at \href{https://github.com/probkit/V-ECE}{https://github.com/probkit/V-ECE}.

\paragraph{Data.}
We use the TabRepo-binary and TabRepo-multiclass benchmarks of CalArena \citep{berta2026calarena}, which provide, for each dataset, the predictions of eight different models (details in \Cref{app:data}). Each (dataset, model) pair has a \emph{calibration split}, of out-of-fold predictions on the training data, and a \emph{test split}, of predictions on held-out data.
We keep the 22 binary and 21 multiclass datasets (3 to 26 classes) whose calibration split has at least 8,000 samples (\Cref{table:datasets}), hence 176 binary and 168 multiclass classifiers~$f$, with test splits of 915 to 10,177 samples.
The calibration error of each classifier is estimated from its test split $\{(f(X_i), Y_i)\}_{i=1}^n$ alone, as in \Cref{alg:CE}. The calibration split is only used to build the ground truth of \Cref{sec:exp:semisynthetic}.

\paragraph{Calibration errors.}
We estimate five calibration errors: $\mathrm{CE}_{f,|\cdot|}$ and $\mathrm{CE}_{f,(\cdot)^2}$ in the binary case, and $\mathrm{CE}_{f,\|\cdot\|_1}$, $\mathrm{CE}_{f,\|\cdot\|_2}$ and $\mathrm{CE}_{f,\|\cdot\|_2^2}$ in the multiclass case.
The two squared errors are proper calibration errors for which the variational estimators use the Brier score.
The three others use the losses of \Cref{cor:lp}.
For a given $\hat g$, all of them are computed from the same out-of-fold predictions $\hat g(f(X_i))$, so the recalibration function is fitted once per classifier, however many calibration errors are reported.
Calibration errors are nonnegative, so we clip at 0 the estimate aggregated over the $m$ folds of \Cref{alg:CE}; only the variational estimators and the debiased ECE of \Cref{sec:exp:semisynthetic} can return negative values.
\Cref{app:clipping} explains why we clip the aggregate rather than the folds, and how clipping interacts with the lower bound.

\subsection{Choosing the recalibration function}
\label{sec:exp:recalibrators}

\paragraph{Protocol.}
On real data the true calibration error is unknown, but \Cref{prop:lowerbound} provides a criterion. Every cross-validated estimate is a lower bound in expectation, so among recalibration functions evaluated on the same folds, a larger estimate means that more of the calibration error was recovered.
We run \Cref{alg:CE} with $m \! = \! 5$ folds on the test split of each classifier, with its real labels, for each candidate $\hat g$.
We rank the candidates by their estimate on each classifier and each calibration error (rank~1 for the largest estimate, tied estimates
sharing their average rank), and average the ranks over classifiers.
We drop the classifiers whose largest estimate of $\mathrm{CE}_{f,|\cdot|}$ (binary) or $\mathrm{CE}_{f,\|\cdot\|_2}$ (multiclass) is below $10^{-3}$: they are essentially calibrated and all estimates are then noise around zero. This leaves 149 binary and 150 multiclass classifiers.

\paragraph{Candidates.}
We consider standard post-hoc calibration methods.
In the binary case these are histogram binning with 10 equal-width (Uniform) or equal-mass (Quantile) bins \citep{zadrozny2001obtaining}, isotonic regression \citep{zadrozny2002transforming}, quadratic scaling \citep{berta2025structured} and Beta calibration \citep{kull2017beta}.
In the multiclass case they are temperature scaling \citep[TS,][]{guo2017calibration} and structured vector and matrix scaling \citep[SVS and SMS,][]{berta2025structured}.
We compare them to WS-CatBoost, a warm-started gradient-boosting recalibrator.
It first fits temperature scaling on the training folds. A CatBoost classifier \citep{prokhorenkova2018catboost} then takes $f(X)$ as its features and the temperature-scaled logits as its initial prediction, and boosts a correction of these logits.
The warm start lets the trees spend their capacity on the miscalibration left by temperature scaling, instead of first relearning the identity map. It therefore needs few trees (at most 100, with early stopping).
Within each training fold, eight CatBoost models are fitted on inner folds, with early stopping on the held-out inner fold, and their predictions are averaged (details in \Cref{app:recalibrators}).

\begin{figure}[t]
    \centering
    \includegraphics[width=0.6\columnwidth]{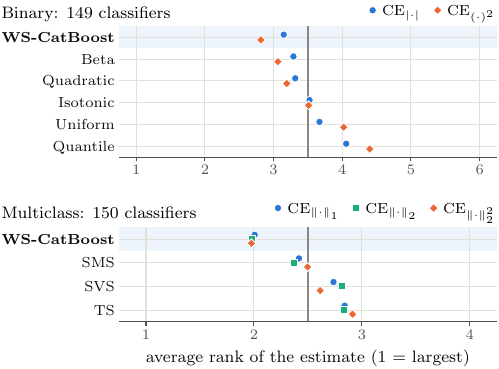}
    \caption{\textbf{Calibration error recovered by each recalibration function.}
    For each classifier and calibration error, recalibration functions are ranked by their cross-validated estimate (rank~1 for the largest, i.e.\ the tightest lower bound), and ranks are averaged over classifiers. The vertical line is the average rank of a method no better than the others.}
    \label{fig:recalibrators}
\end{figure}

\paragraph{Results.}
WS-CatBoost has the best average rank for all five calibration errors (\Cref{fig:recalibrators}), and ranks 2.0 on the three multiclass ones against 2.4 to 2.5 for SMS, the runner-up (\Cref{table:recalibrators}).
How much the choice of $\hat g$ matters depends on the calibration error, as predicted by \Cref{thm:alignment}.
For $\mathrm{CE}_{f,|\cdot|}$, all binary recalibration functions rank close to the no-preference level: recovering an $L_1$ calibration error only requires the side of $f(X)$ on which $C$ lies, and even histogram binning often gets it right.
For $\mathrm{CE}_{f,(\cdot)^2}$, which requires $\hat g \approx C$, the ranks spread further and the two binning methods come last.
WS-CatBoost takes about 1.6 (binary) and 2.6 (multiclass) seconds per 1,000 samples on a 10-core laptop CPU, which remains modest for an evaluation metric.
We therefore use WS-CatBoost as the default recalibration function of our estimator, and call the resulting estimator \emph{V-ECE}.

\subsection{Semi-synthetic benchmark with known calibration error}
\label{sec:exp:semisynthetic}

\paragraph{Ground truth.}
Measuring the accuracy of an estimator requires $C = \mathbb{E}[Y | f(X)]$, which real data never reveal, while fully synthetic data rarely reproduce how real models are miscalibrated.
We therefore build a semi-synthetic task from each CalArena experiment.
We fit TabICLv2 \citep{qu2026tabiclv2}, a tabular foundation model used by none of the estimators, on the calibration split.
The inputs are the predictions $f(X)$ (the probability of class~1 in the binary case, the probability vector otherwise) and the targets are the observed labels, with a context of at most 10,000 random samples.
Applied to the test split, the fitted model $\tilde g$ defines a semi-synthetic dataset $\{(f(X_i), \tilde Y_i)\}_{i=1}^n$, where each new label $\tilde Y_i$ is drawn from the categorical distribution with probabilities $\tilde g(f(X_i))$.
Since these probabilities depend on $X_i$ only through $f(X_i)$, the new labels satisfy $\mathbb{E}[\tilde Y_i \mid f(X_i)] = \tilde g(f(X_i))$ exactly: $\tilde g$ is the true calibration map of this dataset.
We therefore take as the true calibration error of $f$ its value under the empirical distribution of the test predictions,
\begin{equation*}
    \mathrm{CE}_{f,d} = \frac{1}{n} \sum_{i=1}^n d\big(f(X_i), \tilde g(f(X_i))\big) \, ,
\end{equation*}
which is computed exactly from $\tilde g$ and does not depend on the sampled labels.
We draw the labels three times independently, which gives three semi-synthetic datasets per classifier.
Every estimator is run separately on each of them, seeing only the test predictions and the new labels, and its error is computed on each draw before being averaged over draws: we thus score the estimates a user would actually obtain from a single dataset.
The calibration maps $\tilde g$ are learned from real outcomes, so the benchmark inherits (to the extent that is possible) the miscalibration patterns of real models.
The true $\mathrm{CE}_{f,|\cdot|}$ of the binary classifiers has median $0.012$ and the true $\mathrm{CE}_{f,\|\cdot\|_1}$ of the multiclass ones has median $0.057$.

\paragraph{Estimators.}
We compare V-ECE with 16 binary and 4 multiclass estimators in three groups, listed in \Cref{table:estimators} and detailed in \Cref{app:estimators}.

(i)~\emph{Plug-in estimators} approximate $C$ and plug it into the distance:
the binned ECE with 15 equal-width (Uniform-ECE) or equal-mass (Quantile-ECE) bins \citep{naeini2015obtaining,guo2017calibration} or a binning scheme determined by isotonic regression (Isotonic-ECE);
the de-biased version of the estimator (Debiased-ECE, \citealp{kumar2019verified}), the smooth-ECE \citep{blasiok2024smooth} and the kernel version of \citet{popordanoska2022consistent} (KDE-ECE).

(ii)~\emph{Variational estimators with an in-sample $\hat g$} use histogram binning (Uniform-IS \& Quantile-IS) or isotonic regression (Isotonic-IS), the latter being the calibration component of the CORP decomposition \citep{dimitriadis2021stable}.

(iii)~\emph{Variational estimators with an out-of-sample~$\hat g$}, that is derivatives of V-ECE with various recalibration functions: histogram binning (Uniform \& Quantile), isotonic regression, temperature scaling (TS), quadratic scaling, Beta calibration (binary) and SMS (multiclass), the kernel regression of KDE-ECE, and the default WS-CatBoost. All of these are applied with 5-fold cross validation (except the leave-one-out KDE) and inherit the lower bound of \Cref{prop:lowerbound}.

Only KDE, TS, SMS and WS-CatBoost apply to multiclass predictions.

\paragraph{Metrics.}
The true calibration error spans two orders of magnitude across classifiers, so absolute errors cannot be averaged as they are, and relative errors explode on nearly calibrated classifiers.
Instead, for each classifier, label draw and calibration error, we divide the absolute error $|\widehat{\mathrm{CE}} - \mathrm{CE}|$ of every estimator by the largest absolute error among the estimators on that draw. We then average this \emph{normalized absolute error} (NAE) over draws and classifiers: an estimator scores 0 if it is always exact and 1 if it is always the worst.
The NAE depends on the pool of estimators, which differs between the binary and the multiclass case.
We therefore also report the pool-independent normalized mean absolute error, $\mathrm{NMAE} = \sum|\widehat{\mathrm{CE}} - \mathrm{CE}| / \sum \mathrm{CE}$ with sums over classifiers and draws, in \Cref{app:exp:results}, together with the share of draws on which each estimator overestimates the truth.
We compare estimators with a two-sided Wilcoxon signed-rank test \citep{wilcoxon1945individual} on their absolute errors averaged over the three draws of each classifier, at level $\alpha = .05$.

\begin{figure}[t]
    \centering
    \includegraphics[width=0.6\columnwidth]{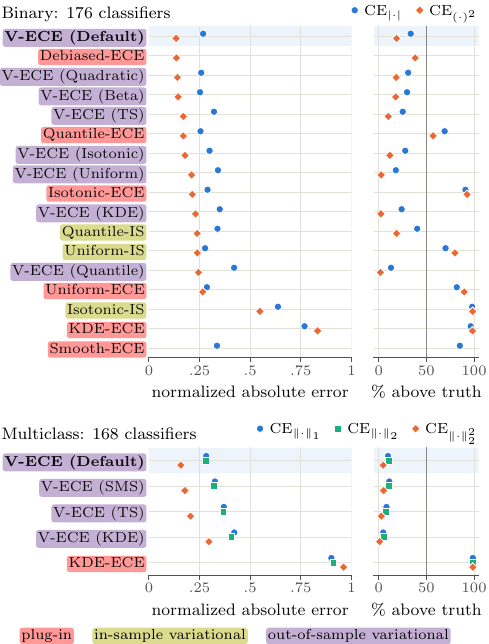}
    \caption{\textbf{Accuracy of calibration error estimators on the semi-synthetic benchmark}, where the true calibration error of each classifier is known.
    \emph{Normalized absolute error}: on each classifier and label draw, the absolute error of an estimator divided by the largest absolute error among the estimators, averaged over draws and classifiers (lower is better; values depend on the pool of estimators, so they are not comparable between the binary (top) and multiclass (bottom) panels).
    \emph{\% above truth}: share of draws on which the estimate exceeds the true calibration error.
    Estimators are sorted by their normalized absolute error on the squared calibration error; Debiased-ECE only defines $\mathrm{CE}_{f,(\cdot)^2}$, and Smooth-ECE, which only defines $\mathrm{CE}_{f,|\cdot|}$, is listed last. Exact values, NMAE and runtimes are in \Cref{table:estimators:binary,table:estimators:multiclass}.}
    \label{fig:accuracy}
\end{figure}

\paragraph{Binary results.}
Our variational estimators are the most accurate in the binary case (\Cref{fig:accuracy}, top).
For $\mathrm{CE}_{f,(\cdot)^2}$, the default V-ECE is the most accurate estimator.
It is statistically tied with V-ECE with quadratic scaling, and significantly more accurate than the 14 other estimators, including the debiased ECE, which is designed for this binary squared error only.
For $\mathrm{CE}_{f,|\cdot|}$, V-ECE with Beta calibration is slightly but significantly ahead of the default V-ECE, which is tied with quadratic scaling, temperature scaling and Quantile-ECE, and significantly more accurate than the 11 other estimators.
That a parametric recalibration function suffices here is consistent with \Cref{thm:alignment}: estimating an $L_1$ calibration error only requires the side of $f(X)$ on which $C$ lies.
Out-of-sample and in-sample estimators err in opposite directions (\Cref{fig:accuracy}, \% above truth): for $\mathrm{CE}_{f,|\cdot|}$, most estimators whose $\hat g$ is fitted on the evaluation points overestimate on 69 to 98\% of the draws, whereas the V-ECE variants do so on 13 to 34\%, through sampling noise only.

\paragraph{Multiclass results.}
Default V-ECE is the most accurate estimator for all three multiclass calibration errors (\Cref{fig:accuracy}, bottom).
All differences are significant, and V-ECE overestimates on only 5 to 10\% of the draws.
The next three estimators are themselves our variational estimator with weaker recalibration functions.
The only existing multiclass estimator, KDE-ECE, is far behind, and it overestimates on 99\% of the draws.
Scoring its kernel regression by excess risk instead (V-ECE (KDE)) removes the overestimation, but the estimate then falls far below the truth (median ratio $\widehat{\mathrm{CE}}/\mathrm{CE}$ of $0.50$ for $\mathrm{CE}_{f,\|\cdot\|_1}$).
The gap between WS-CatBoost and the parametric recalibration functions is wider than in the binary case.
TS and SMS recover a median $55$--$67\%$ of $\mathrm{CE}_{f,\|\cdot\|_1}$ and $\mathrm{CE}_{f,\|\cdot\|_2}$, against $71$--$72\%$ for WS-CatBoost, and $4$--$14\%$ of $\mathrm{CE}_{f,\|\cdot\|_2^2}$, against $30\%$.
In the multiclass case, the nonparametric correction that WS-CatBoost adds to temperature scaling pays off.

\paragraph{$L_p$ versus proper calibration errors.}
Across settings, $L_p$ calibration errors are recovered much better than proper ones. The median ratio $\widehat{\mathrm{CE}}/\mathrm{CE}$ of V-ECE is $0.84$ for $\mathrm{CE}_{f,|\cdot|}$ and $0.72$ and $0.71$ for $\mathrm{CE}_{f,\|\cdot\|_1}$ and $\mathrm{CE}_{f,\|\cdot\|_2}$, but only $0.37$ and $0.30$ for the two squared errors (\Cref{table:estimators:binary,table:estimators:multiclass}), and the same holds for every out-of-sample variational estimator.
This is the behavior described by \Cref{thm:alignment}. The bias of an $L_p$ estimator only depends on the direction of $\hat g \circ f - f$. The bias of a proper estimator is the full recalibration error $\mathbb{E}[d_\ell(\hat g\circ f(X), C)]$, which dominates when the calibration error itself is small.

\paragraph{Ranking models.}
In terms of ordering the eight models of each dataset by their calibration errors (\Cref{app:ranking}), V-ECE is the best estimator for the three multiclass errors, while in-sample estimators are better for the binary $\mathrm{CE}_{f,|\cdot|}$.

\paragraph{Summary.}
V-ECE is the most accurate estimator for the three multiclass calibration errors.
In the binary case, it is the most accurate estimator for $\mathrm{CE}_{f,(\cdot)^2}$, ahead of the debiased ECE, and statistically tied for second, behind its Beta variant for $\mathrm{CE}_{f,|\cdot|}$.
The estimators that come close are restricted to two classes, and the debiased ECE to a single calibration error.
WS-CatBoost is the only recalibration function at the top in both binary and multiclass settings, which is why we recommend it as the default.
Moreover, the loss only enters when evaluating the out-of-fold predictions, so the same $\hat g$ yields, at no additional cost, any calibration error covered by \Cref{thm:general}: other $L_p$ norms, weighted or truncated divergences (\Cref{app:general:distances}), or the over- and under-confidence decomposition of \Cref{prop:overunder}.

\paragraph{Limitations.}
The semi-synthetic ground truth depends on TabICL. The map $\tilde g$ inherits its estimation noise, which can make the semi-synthetic calibration map rougher, and thus harder to recover, than the real one. Its miscalibration patterns may also differ from those of other model classes or data modalities.
Moreover, a lot of the binary TabRepo classifiers are well calibrated, so the binary results mostly describe the regime of small calibration errors, where all estimators are least accurate.

\section{Conclusion}

Letting a proper score depend on the prediction extends the variational estimator from Bregman divergences to every convex divergence on the simplex, including the $L_1$ and $L_2$ calibration errors that are usually reported, with closed-form losses and the cross-validation lower bound unaffected.
The bias of the resulting estimator depends on the direction of $\hat g \circ f - f$ rather than on its magnitude, which makes $L_p$ calibration errors easier to estimate than proper ones, as observed in our experiments.
With a warm-started gradient-boosting recalibrator, V-ECE is the most accurate multiclass estimator and among the most accurate binary ones on a semi-synthetic benchmark built from real classifiers, and obtains all calibration errors from a single fitted recalibration function.

\section*{Acknowledgements}

Authors acknowledge funding from the European Union (ERC-2022-SYG-OCEAN-101071601). Views and opinions expressed are however those of the author(s) only and do not necessarily reflect those of the European Union or the European Research Council Executive Agency. Neither the European Union nor the granting authority can be held responsible for them.

This publication is part of the Chair "Markets and Learning", supported by Air Liquide, BNP PARIBAS ASSET MANAGEMENT Europe, EDF, Orange and SNCF, sponsors of the Inria Foundation.

This work has also received support from the French government, managed by the National Research Agency, under the France 2030 program with the reference "PR[AI]RIE-PSAI" (ANR-23-IACL-0008).

\bibliography{references}

\clearpage
\appendix
\crefalias{section}{appendix}
\thispagestyle{empty}

\numberwithin{figure}{section}
\numberwithin{table}{section}

\onecolumn
\aistatstitle{Supplementary Materials}

\etocdepthtag.toc{mtappendix}

\begingroup
  \etocsettagdepth{mtchapter}{none}
  \etocsettagdepth{mtappendix}{subsection}
  \setlength{\parskip}{0pt}

  \etocsettocstyle{%
    \noindent{\large Contents}\par\vspace{0.8em}}{}

  \etocsetstyle{section}
    {}
    {\addvspace{1.3ex}}
    {\noindent
     \etoclink{\color{tocblue}\makebox[1.5em][l]{\etocthenumber}\etocthename}%
     \hfill\makebox[2em][r]{\etocthepage}\par}
    {}

  \etocsetstyle{subsection}
    {\addvspace{0.3ex}}
    {}
    {\noindent\hspace*{1.5em}%
     \etoclink{\color{tocblue}\makebox[2.4em][l]{\etocthenumber}\etocthename}%
     \nobreak\leaders\hbox to 0.9em{\hss.\hss}\hfill
     \makebox[2em][r]{\etocthepage}\par\addvspace{0.35ex}}
    {}

  \tableofcontents
\endgroup

\clearpage

\section{Intuition for the $L_1$ calibration error}
\label{app:intuition}

This section explains, in the binary case, how the divergence $D_{f(X)}(z) = |z - f(X)|$ is arrived at, and why the resulting estimator measures $\mathbb{E}[|f(X) - C|]$.
We write $f$ for $f(X)$ throughout.

\subsection{The entropy function}

\Cref{thm:general} asks for a function $D_{f}$ that is convex, vanishes at the prediction, and admits $0$ as a subgradient there; the quantity it then returns is $\mathbb{E}[D_{f}(C)]$.
Choosing $D_{f}(z) = |z-f|$ is the same as choosing the entropy $H_{f}(z) = -|z-f|$.
\Cref{fig:intuition:entropy} shows why this choice is forced if one wants the $L_1$ calibration error, and why no fixed proper loss can produce it.

\begin{figure}[ht!]
    \centering
    \includegraphics[width=0.97\textwidth]{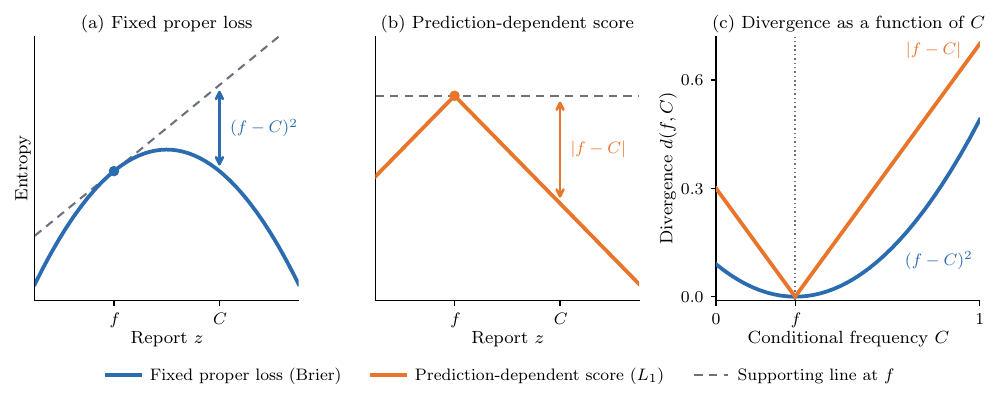}
    \caption{\textbf{The divergence of a proper loss as a gap under a supporting line.}
    (a) For a fixed proper loss with entropy $H$, the divergence $d_\ell(f, C)$ is the vertical gap at $z=C$ between the tangent to $H$ at $z=f$ and the graph of $H$. Where $H$ is twice differentiable this gap closes quadratically, and for the Brier entropy $H(z)=z(1-z)$ it is exactly $(f-C)^2$.
    (b) The prediction-dependent entropy $H_{f}(z) = -|z-f|$ has a kink at $z=f$. The supporting line with slope $0$ is horizontal, and the gap at $z=C$ is $|f-C|$.
    (c) The two divergences as functions of $C$: linear against quadratic near the prediction.}
    \label{fig:intuition:entropy}
\end{figure}

Concavity of $H$ is what makes the gap in \Cref{fig:intuition:entropy}(a) nonnegative, and is therefore what makes a proper loss induce a genuine divergence.
It also fixes the rate at which the gap closes.
When $H$ is differentiable at $f$, the supporting line is tangent to the graph, so the gap is $o(|f-C|)$ as $C \to f$.
A gap that grows \emph{linearly} in $|f-C|$ requires the supporting line to separate from the graph at a linear rate, which happens only where $H$ has a kink.
A single concave function on $[0,1]$ can have kinks, but only on a set of measure zero, which is the content of \Cref{prop:no:go}; letting $H$ depend on the prediction places a kink exactly where it is needed, at $z=f$, for every value of $f$.

The choice of subgradient at the kink also matters.
The superdifferential of $H_f$ at $f$ is the whole interval $[-1,1]$, and we select $0$.
This makes the supporting line horizontal, so that $\ell_{f}(f, Y) = 0$ for every $Y$: the risk of the model itself vanishes and the estimator reduces to the negative risk of the recalibrated model alone.

\subsection{Geometrical interpretation of the estimator}

Fix a prediction $f$ and suppose the recalibration function reports the value $z$.
Since $\ell_{f}(z,Y) = \mathds{1}_{z \neq f}\, \sign(z-f)\,(f - Y)$, the population value of the estimator at that prediction is
$$
-\mathbb{E}_{Y \sim C}\big[ \ell_{f}(z, Y) \big] = \sign(z - f)\, (C - f) \, ,
$$
a step function of $z$ that jumps at $z = f$, plotted in \Cref{fig:intuition:objective}(a).

\begin{figure}[ht!]
    \centering
    \includegraphics[width=0.74\textwidth]{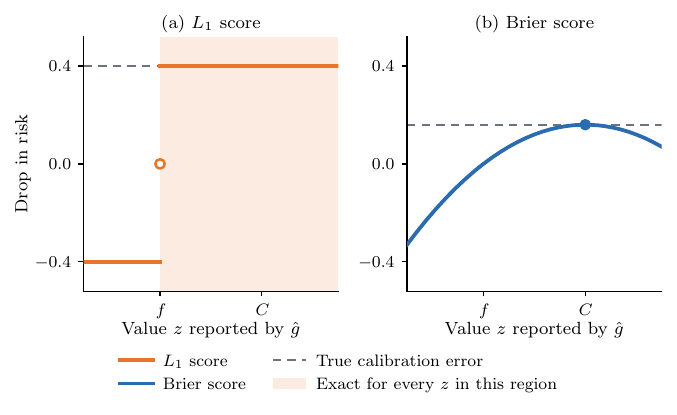}
    \caption{\textbf{What the estimator returns as a function of the value reported by $\hat g$}, at a fixed prediction $f$ with $C > f$.
    (a) With the $L_1$ score the value is a step: it equals the true $|f - C|$ for every $z$ on the same side of $f$ as $C$, and $-|f-C|$ on the other side.
    (b) With the Brier score the value is a parabola peaked at $z = C$, so only the exact conditional probability attains $(f-C)^2$.}
    \label{fig:intuition:objective}
\end{figure}

Two things follow.
First, the supremum over $z$ is $|C - f|$, which is \Cref{cor:lp} read at a single prediction: maximizing over the recalibration function recovers the $L_1$ distance between $f$ and $C$, and averaging over $f(X)$ gives $\mathrm{CE}_{f,\|\cdot\|_1}$.
Second, the supremum is attained on a half-line rather than at a point, which is \Cref{thm:alignment}(ii): the estimator needs the sign of $C - f$, not the value of $C$.
\Cref{fig:intuition:objective}(b) shows the same object for the Brier score, where the optimum is attained only at $z = C$.

The step function also gives a concrete reading of the estimator on a single observation.
If $\hat g \circ f(X) > f(X)$ then $-\ell_{f(X)}(\hat g \circ f(X), Y) = Y - f(X)$, and if $\hat g \circ f(X) < f(X)$ it is $f(X) - Y$.
The estimator is therefore nothing but an average of the residuals $Y_i - f(X_i)$, each multiplied by $\pm 1$ according to which side of $f(X_i)$ the recalibration function falls on.
Conditionally on $f(X)$ the residual has mean $C - f(X)$, so as soon as the sign is correct the conditional mean of each term is $|C - f(X)|$, which is exactly the quantity we want to average.
The role of $\hat g$ is only to supply that sign.

\section{Proofs}
\label{app:proofs}

\subsection{Proper losses do not induce $L_p$ distances}

\begin{proposition} \label[proposition]{prop:no:go}
Let $k \ge 2$ and $p \ge 1$. There is no proper loss $\ell$ on $\Delta_k$ whose induced divergence satisfies
$$
d_\ell(u,v) = \| u - v \|_p \quad \text{for all } u, v \in \relint \Delta_k \, .
$$
\end{proposition}

\begin{proof}
Suppose such an $\ell$ exists, with concave entropy $H$ as in \eqref{eq:proper:from:H}.
Fix $u \in \relint \Delta_k$ and a direction $w \ne 0$ in the tangent space $\{ w \in \mathbb{R}^k : \sum_i w_i = 0\}$, and let $t > 0$ be small enough that $u \pm t w \in \relint \Delta_k$.
Adding the two instances $v = u + tw$ and $v = u - tw$ of \eqref{eq:bregman}, the linear terms $\langle \delta H(u), \pm tw \rangle$ cancel, so
\begin{align*}
&d_\ell(u, u+tw) + d_\ell(u, u-tw) = 2 H(u) - H(u + tw) - H(u - tw) \, .
\end{align*}
If $d_\ell = \| \cdot \|_p$, the left-hand side equals $2 t \| w \|_p$, hence
\begin{equation} \label{eq:no:go:kink}
    \frac{2H(u) - H(u+tw) - H(u-tw)}{t} = 2 \|w\|_p > 0
\end{equation}
for all small $t>0$ and every $u \in \relint \Delta_k$. Now, $H$ is concave and finite on the open (relative to $\Delta_k$) convex set $\relint \Delta_k$, hence differentiable outside a set of measure zero \citep{rockafellar1970convex}; pick any $u$ at which $H$ is differentiable.
At such a $u$, the left-hand side of \eqref{eq:no:go:kink} tends to $-\langle \nabla H(u), w\rangle - \langle \nabla H(u), -w\rangle = 0$ as $t \to 0^+$, contradicting \eqref{eq:no:go:kink}.
\end{proof}

\subsection{Proof of \Cref{thm:general}}

We restate the result for convenience.

\begin{restated}{thm:general}
Let $\Delta_k$ be the probability simplex in $\mathbb{R}^k$. Let $f(X) \in \Delta_k$ be a baseline prediction. Let $D_{f(X)}: \Delta_k \to \mathbb{R}_{\geq 0}$ be a convex function minimized at $f(X)$, such that $D_{f(X)}(f(X)) = 0$.
Let $\nabla D_{f(X)}$ be a subgradient of $D_{f(X)}$ satisfying $\nabla D_{f(X)}(f(X)) = 0$. Then, for $z \in \Delta_k$ and a one-hot encoded vector $Y \in \{e_1, \dots, e_k\}$, the function
\begin{equation*}
    \ell_{f(X)}(z, Y) := -D_{f(X)}(z) + \langle - \nabla D_{f(X)}(z), Y - z \rangle \, ,
\end{equation*}
is a proper loss satisfying
\begin{equation*}
    \mathbb{E}_{X}[D_{f(X)}(C)] = \E[\ell_{f(X)}(f(X), Y)]- \min_{g\in\mathcal{H}} \E[\ell_{f(X)}(g\circ f(X), Y)] \, ,
\end{equation*} with $\mathcal{H}$ the set of all measurable functions from $\Delta_k$ to $\Delta_k$ and $C = \E[Y|f(X)]$.
\end{restated}

\begin{proof}
We first check that $\ell_{f(X)}$ is proper. For $z, v \in \Delta_k$, linearity of the expectation and the subgradient inequality for the convex function $D_{f(X)}$ give
$$
\mathbb{E}_{Y \sim v}[\ell_{f(X)}(z, Y)]
= -D_{f(X)}(z) - \langle v - z, \nabla D_{f(X)}(z) \rangle
\;\geq\; -D_{f(X)}(v)
= \mathbb{E}_{Y \sim v}[\ell_{f(X)}(v, Y)] \, ,
$$
so reporting the true distribution $v$ is optimal. The two anchoring conditions $D_{f(X)}(f(X)) = 0$ and $\nabla D_{f(X)}(f(X)) = 0$ make the loss vanish identically at the prediction itself, $\ell_{f(X)}(f(X), Y) = 0$ for every~$Y$.

Now condition on $f(X)$ and let $Y \sim C$. Reporting $C$ costs
$$
\mathbb{E}_{Y \sim C}[\ell_{f(X)}(C, Y)] = -D_{f(X)}(C) - \langle C - C, \nabla D_{f(X)}(C) \rangle = -D_{f(X)}(C) \, ,
$$
so the excess risk of the prediction $f(X)$ over the report $C$ is exactly $D_{f(X)}(C)$. It remains to take the expectation over $X$ and to recall that, for a proper loss, the conditional expectation $g^\star \circ f(X) = C$ attains the infimum of $\mathbb{E}[\ell_{f(X)}(g \circ f(X), Y)]$ over measurable $g$ (see Lemma A.6 of \citealp{berta2025rethinking}), which is the stated identity.
\end{proof}

\subsection{Cross-validation gives a lower bound}

\begin{proposition} \label[proposition]{prop:lowerbound}
Let $(f(X_i), Y_i)_{i=1}^n$ be i.i.d.\ copies of $(f(X), Y)$ and $\mathcal{I}_1, \dots, \mathcal{I}_m$ partition $\{1, \dots, n\}$ independently of the data.
For each $j$, let $\hat g_j$ be any measurable function of $\{(f(X_i), Y_i) : i \notin \mathcal{I}_j\}$.
Let $\{D_u\}_{u \in \Delta_k}$ satisfy the assumptions of \Cref{thm:general}, with associated losses $\ell_u$. Then
$$
\mathbb{E}\Big[ -\tfrac{1}{n} \sum_{j=1}^m \sum_{i \in \mathcal{I}_j} \ell_{f(X_i)}\big(\hat g_j \circ f(X_i), Y_i\big) \Big] \le \mathrm{CE}_{f,D} \, ,
$$
where $\mathrm{CE}_{f,D} = \mathbb{E}[D_{f(X)}(C)]$. The same holds for the two-term estimator of \Cref{sec:background} computed fold-wise with any proper loss $\ell$.
\end{proposition}

\begin{proof}
Fix a fold $j$ and an index $i \in \mathcal{I}_j$.
By construction $\hat g_j$ is a function of $\{(f(X_i'), Y_{i'}) : i' \notin \mathcal{I}_j\}$ only, and the partition does not depend on the data, so $\hat g_j$ is independent of $(f(X_i), Y_i)$.
Conditioning on $\hat g_j$ and using that $(f(X_i), Y_i)$ has the same law as $(f(X), Y)$,
$$
\mathbb{E}\big[ \ell_{f(X_i)}(\hat g_j \circ f(X_i), Y_i) \,\big|\, \hat g_j \big] = \Psi(\hat g_j), \qquad \Psi(g) \coloneqq \mathbb{E}\big[ \ell_{f(X)}(g \circ f(X), Y) \big] \, .
$$
Since $\hat g_j$ takes values in $\mathcal{H}$, we have $\Psi(\hat g_j) \ge \inf_{g \in \mathcal{H}} \Psi(g)$ pointwise, and \Cref{thm:general} identifies the right-hand side as $-\mathrm{CE}_{f,D}$. Taking expectations,
$$
\mathbb{E}\big[ -\ell_{f(X_i)}(\hat g_j \circ f(X_i), Y_i) \big] = -\mathbb{E}[\Psi(\hat g_j)] \le \mathrm{CE}_{f,D} \, .
$$
This holds for every one of the $n$ pairs $(j, i)$ with $i \in \mathcal{I}_j$, and the folds partition $\{1,\dots,n\}$, so averaging the $n$ inequalities with weights $1/n$ gives the claim.

For \eqref{eq:CE_hat} with a fixed proper loss $\ell$, the same argument applies to $\Psi(g) = \mathbb{E}[\ell(g\circ f(X),Y)]$, using $\mathbb{E}[\ell(f(X),Y)] - \inf_{g\in\mathcal{H}}\Psi(g) = \mathrm{CE}_{f,d_\ell}$ from \eqref{eq:VariationalCalibration} and noting that the term $\frac1n\sum_i \ell(f(X_i),Y_i)$ is unbiased for $\mathbb{E}[\ell(f(X),Y)]$.
\end{proof}

\subsection{Proof of \Cref{cor:lp}}

\begin{proof}
$D_{f(X)}(z) = \|z - f(X)\|_p$ is convex, nonnegative and vanishes at $z = f(X)$; among the subgradients of $D_{f(X)}$ at $f(X)$ we select $0$, so \Cref{thm:general} applies and gives $\mathrm{CE}_{f,\|\cdot\|_p} = \mathbb{E}[\|f(X) - C\|_p] = -\mathbb{E}[\ell_{f(X)}(g^\star \circ f(X), Y)]$.
It remains to simplify \eqref{eq:general:loss}. For $z \ne f(X)$,
$$
\ell_{f(X)}(z, Y) = \langle \nabla_z \| z - f(X) \|_p , z-Y \rangle - \| z - f(X) \|_p \, ,
$$
and since $\|\cdot\|_p$ is positively homogeneous of degree one, Euler's identity gives $\| z - f(X) \|_p = \langle \nabla_z\| z - f(X) \|_p , z - f(X) \rangle$, so that
$$
\ell_{f(X)}(z, Y) = \langle \nabla_z \| z - f(X) \|_p , f(X) - Y \rangle \, .
$$
The expression for $\nabla_z \|z - f(X)\|_p$ follows from a direct computation.
\end{proof}

\subsection{Proof of \Cref{thm:alignment}}

\begin{proof}
Fix $p \ge 1$ with conjugate exponent $q$, so that $1/p + 1/q = 1$, and write $w_g \coloneqq \nabla_z \|z - f(X)\|_p \big|_{z = g \circ f(X)}$ for the vector against which the loss of \Cref{cor:lp} measures the residual $f(X) - Y$.

By \Cref{cor:lp}, $-\ell_{f(X)}(z, Y) = \langle \nabla_z\|z - f(X)\|_p, Y - f(X)\rangle$. Since $w_g$ depends on $X$ only through $f(X)$, conditioning on $f(X)$ replaces $Y$ by $C$, and with $a_g = g \circ f(X) - f(X)$, $b = C - f(X)$,
$$
\mathrm{CE}_{f,\|\cdot\|_p}(g) = \mathbb{E}\big[\langle w_g, Y - f(X)\rangle\big] = \mathbb{E}\big[\langle w_g, b\rangle\big] = \mathbb{E}\big[\|b\|_p\, \rho_p(a_g, b)\big] \, ,
$$
where the summand is read as $0$ when $b = 0$. Subtracting this from $\mathrm{CE}_{f,\|\cdot\|_p} = \mathbb{E}[\|b\|_p]$ is \eqref{eq:alignment}.

The non-negativity in \eqref{eq:alignment} is the variational inequality of \Cref{thm:general}: the score $\ell_{f(X)}$ is proper, so $g^\star$ minimizes the risk and $\mathrm{CE}_{f,\|\cdot\|_p}(g) \le \mathrm{CE}_{f,\|\cdot\|_p}(g^\star) = \mathrm{CE}_{f,\|\cdot\|_p}$ for every $g$.

For the two consequences we need the equality case of H\"older's inequality, and hence the norm of $w_g$. For $a \neq 0$ we have $\nabla\|a\|_p = \sign(a) \odot |a|^{p-1} / \|a\|_p^{p-1}$, and when $p > 1$ the identity $(p-1)q = p$ gives
$$
\big\| \nabla \|a\|_p \big\|_q^q = \frac{\sum_i |a_i|^{(p-1)q}}{\|a\|_p^{(p-1)q}} = \frac{\sum_i |a_i|^p}{\|a\|_p^p} = 1 \, ;
$$
when $p = 1$ the gradient is $\sign(a)$, whose sup-norm is $1$. The indicator in \eqref{eq:lp:loss} sets the vector to zero when $a = 0$, so in every case $\|w_g\|_q \le 1$, and H\"older's inequality $|\langle w_g, b\rangle| \le \|w_g\|_q \|b\|_p \le \|b\|_p$ is what gives the range $\rho_p \in [-1, 1]$ used in the definition.

Since $a \mapsto \nabla\|a\|_p$ is positively homogeneous of degree zero, $\rho_p(\lambda a, b) = \rho_p(a, b)$ for every $\lambda > 0$, which is (i). For (ii), take first $p = 1$: the summand of the bias is $\|b\|_1 - \langle\sign(a), b\rangle = \sum_i |b_i|\,(1 - \sign(a_i)\sign(b_i))$, and the $i$-th term vanishes exactly when $\sign(a_i) = \sign(b_i)$ or $b_i = 0$. When $p \in (1, \infty)$, equality in H\"older with $\|w_g\|_q = 1$ forces $|w_{g,i}|^q$ to be proportional to $|b_i|^p$ with matching signs, that is $|a_i|^p \propto |b_i|^p$ with $\sign(a_i) = \sign(b_i)$, which says precisely that $a_g$ is a positive multiple of $b$.
\end{proof}

The bound for $p = 1$ stated in \Cref{sec:alignment} is a direct reading of the computation in (ii). If $\hat g$ is independent of the evaluation sample, the theorem applies conditionally on $\hat g$, and the bias equals $\sum_i \mathbb{E}[|b_i|(1 - \sign(a_i)\sign(b_i))]$ with $a = \hat g \circ f(X) - f(X)$. The factor $1 - \sign(a_i)\sign(b_i)$ takes values in $\{0, 1, 2\}$ and is zero exactly when the two signs agree and are nonzero, so it is at most $2\,\mathds{1}_{\sign(a_i) \neq \sign(b_i)}$, which is the announced inequality.

For comparison, the analogous computation with a fixed proper loss $\ell$ gives \eqref{eq:proper:bias}: conditioning on $f(X)$, the expected loss of any report $z$ under $Y \sim C$ is $d_\ell(z, C) + e_\ell(C)$, so subtracting the risks of $f(X)$ and of $g \circ f(X)$ cancels the entropy terms and leaves $\mathbb{E}[d_\ell(f(X), C)] - \mathbb{E}[d_\ell(g \circ f(X), C)]$, whose first term is $\mathrm{CE}_{f,d_\ell}$.

\subsection{Proof of \Cref{prop:overunder}}

\begin{proof}
Write $f$ for $f(X)$ and $s = \sign(f - 1/2)$, which lies in $\{-1, +1\}$ almost surely since $\mathbb{P}(f = 1/2) = 0$.

The function $D^{+}_{f}(z) = (s(f - z))_+$ is the composition of the convex map $(\cdot)_+$ with an affine one, hence convex; it is nonnegative and vanishes at $z = f$, so $f$ is a minimizer. Its subdifferential is $\{-s\}$ where $s(f - z) > 0$, $\{0\}$ where $s(f - z) < 0$, and the segment between $0$ and $-s$ at $z = f$. The selection $\nabla D^{+}_{f}(z) = -s\,\mathds{1}_{s(f - z) > 0}$ is therefore a valid subgradient with $\nabla D^{+}_{f}(f) = 0$, and the same holds for $D^{-}$ with $\nabla D^{-}_{f}(z) = s\,\mathds{1}_{s(z - f) > 0}$. Both families thus satisfy the hypotheses of \Cref{thm:general}.

Substituting into the general loss, for the over-confidence side,
$$
\ell^{+}_{f}(z, y) = -s(f - z)\,\mathds{1}_{s(f-z)>0} + s(y - z)\,\mathds{1}_{s(f-z)>0} = s\,(y - f)\,\mathds{1}_{s(f-z)>0} \, ,
$$
and symmetrically $\ell^{-}_{f}(z, y) = -s\,(y - f)\,\mathds{1}_{s(z-f)>0}$, which are the announced forms. \Cref{thm:general} then gives $\mathrm{CE}^{\pm}_{f,|\cdot|} = \mathbb{E}[D^{\pm}_{f}(C)] = -\mathbb{E}[\ell^{\pm}_{f}(g^\star \circ f(X), Y)]$. One can also see this directly: conditioning on $f(X)$ turns $y - f$ into $C - f$, and since $g^\star \circ f(X) = C$ the indicator becomes $\mathds{1}_{s(f - C) > 0}$, so that $-\mathbb{E}[\ell^{+}_{f}(g^\star \circ f(X), Y)] = \mathbb{E}[(s(f - C))_+] = \mathbb{E}[D^{+}_{f}(C)]$.

Finally, $(a)_+ + (-a)_+ = |a|$ for every real $a$; with $a = s(f - C)$ and $|s| = 1$ this reads $D^{+}_{f}(C) + D^{-}_{f}(C) = |f - C|$, and taking expectations gives $\mathrm{CE}^{+}_{f,|\cdot|} + \mathrm{CE}^{-}_{f,|\cdot|} = \mathrm{CE}_{f,|\cdot|}$.
\end{proof}

\subsection{No unbiased estimator exists}
\label{app:conservative}

Consider the simplest possible setting: binary labels, a model that makes the single prediction $f \equiv 1/2$, so that $C = \theta \coloneqq \mathbb{P}(Y = 1)$ and $\mathrm{CE}_{f,|\cdot|} = |\theta - 1/2|$, and a sample $Y_1, \dots, Y_n$ of i.i.d.\ $\mathrm{Bernoulli}(\theta)$ labels. An estimator is any measurable, possibly randomized, map $T : \{0,1\}^n \to \mathbb{R}$.

\begin{proposition} \label[proposition]{prop:toy}
The functions $\theta \mapsto \mathbb{E}_\theta[T]$, as $T$ ranges over all estimators, are exactly the polynomials of degree at most $n$. In particular no estimator of $\mathrm{CE}_{f,|\cdot|} = |\theta - 1/2|$ is unbiased, whatever the sample size.
\end{proposition}

This is a classical fact of the theory of unbiased estimation: in the Bernoulli model with $n$ observations, the functions of $\theta$ that admit an unbiased estimator are exactly the polynomials of degree at most $n$ \citep{halmos1946theory}. We include the short argument for completeness.

\begin{proof}
The count $S = \sum_i Y_i$ is sufficient for $\theta$, so replacing $T$ by $\mathbb{E}[T \mid S]$, which does not depend on $\theta$, leaves $\mathbb{E}_\theta[T]$ unchanged; we may therefore assume that $T(y) = t_{|y|}$ depends on $y$ only through the number of ones. Then
$$
\mathbb{E}_\theta[T] = \sum_{j=0}^{n} t_j \binom{n}{j} \theta^j (1 - \theta)^{n-j}
$$
is a polynomial of degree at most $n$, and randomizing $T$ only averages such polynomials. Conversely, the $n + 1$ Bernstein polynomials $\binom{n}{j}\theta^j(1-\theta)^{n-j}$ are linearly independent in the $(n+1)$-dimensional space of polynomials of degree at most $n$, hence form a basis of it, so every such polynomial is the expectation of the estimator whose values $t_j$ are its Bernstein coefficients. Since $\theta \mapsto |\theta - 1/2|$ is not differentiable at $1/2$, it is not a polynomial, and no estimator has it as its expectation.
\end{proof}

\section{Algorithmic procedure}
\label{app:algorithm}
We summarize the proposed algorithm as well as the cross-validation strategy in \Cref{alg:CE}.
\begin{algorithm}[h!]
\caption{Computing $\widehat{\mathrm{CE}}$ with a prediction-dependent proper score.}
\label{alg:CE}
\begin{algorithmic}[1]
\STATE {\bfseries Input:} Data $\{(f(X_i), Y_i)\}_{i=1}^n$, number of folds $m \ge 2$, family of convex divergences $\{D_u\}_{u \in \Delta_k}$ (e.g. $D_u = \|\cdot - u\|_p$), classification method.
\STATE \textbf{Partition the data:} Randomly divide $\{1, \dots, n\}$ into $m$ approx.\ equal-sized folds
$\{\mathcal{I}_1, \dots, \mathcal{I}_m\}$.

\FOR{$j = 1$ to $m$}
    \STATE \textbf{Define folds:} $\ds \mathcal{I}_{\text{val}}^{(j)} = \mathcal{I}_j,
    \quad
    \mathcal{I}_{\text{tr}}^{(j)} = \{1, \dots, n\} \setminus \mathcal{I}_j.$

    \STATE \textbf{Train classifier:} Fit a classifier $g^{(j)}$ on $\{(f(X_i), Y_i) \mid i \in \mathcal{I}_{\text{tr}}^{(j)}\}$ using the specified method.
    \STATE \textbf{Evaluate on validation fold:} Using $i \in \mathcal{I}_{\text{val}}^{(j)}$, compute
    \[
    \widehat{\mathrm{CE}}^{(j)} = - \cfrac{1}{|\mathcal{I}_{\text{val}}^{(j)}|}\sum_{i\in \mathcal{I}_{\text{val}}^{(j)}} \ell_{f(X_i)}\big(g^{(j)}(f(X_i)), Y_i\big) , \quad \ell_u(z, Y) = -D_u(z) - \langle Y - z, \nabla D_u(z)\rangle .
    \]
\ENDFOR

\STATE \textbf{Aggregate across folds:} $\ds \widehat{\mathrm{CE}} = \sum_{j=1}^m \cfrac{|\mathcal{I}_j|}{n}\widehat{\mathrm{CE}}^{(j)}$.
\end{algorithmic}
\end{algorithm}

\section{Estimating general convex divergences}
\label{app:general:distances}

\Cref{thm:general}, proved in \Cref{app:proofs}, shows that general divergences $d(u, z)$ on the simplex admit a variational form whenever all functions $D_u(z) = d(u, z)$ are convex and minimized at $u$ with $D_u(u) = 0$.
Beyond the $L_p$ norms of \Cref{cor:lp} and the one-sided distances of \Cref{sec:overunder}, admissible choices include, for instance, weighted norms $D_u(z) = \|W(z-u)\|_p$ for a fixed matrix $W$, class-dependent weightings $D_u(z) = \sum_i w_i |z_i - u_i|$ with $w_i \ge 0$, and truncated distances $D_u(z) = (\|z-u\|_p - \tau)_+$ for a tolerance $\tau \ge 0$, which ignore miscalibration below a level deemed practically irrelevant.
Note that the construction requires $D_u$ to be convex \emph{in its second argument only}; no joint convexity or symmetry is needed, which is what makes the one-sided distances of \Cref{sec:overunder} admissible.

\section{Experimental details}
\label{app:exp:details}

This section details the experiments of \Cref{sec:experiments}. All of them can be reproduced with the code provided (\href{https://github.com/probkit/V-ECE}{https://github.com/probkit/V-ECE}), which also contains the per-classifier results.

\subsection{Data}
\label{app:data}

We use the TabRepo-binary and TabRepo-multiclass benchmarks of CalArena \citep{berta2026calarena}, which collect the predictions of TabRepo \citep{tabrepo} for eight AutoGluon model families \citep{erickson2020autogluon} (CatBoost, ExtraTrees, LightGBM, a linear model, NeuralNetFastAI, NeuralNetTorch, RandomForest and XGBoost) on the first outer fold of each dataset, together with the configuration of each model.
The calibration split holds the out-of-fold predictions of the bagged models on the training data, and the test split the predictions on held-out data (10\% of the samples).
Of the 104 binary and 65 multiclass datasets of these benchmarks, we keep the 22 binary and 21 multiclass datasets whose calibration split has at least 8,000 samples, so that the ground truth of \Cref{app:groundtruth} can be fitted accurately (\Cref{table:datasets}). This gives $176 + 168 = 344$ classifiers~$f$.
TabRepo predictions often contain exact zeros, for instance from tree ensembles, which break the recalibration functions that operate on logits. Before any recalibration function is fitted, we therefore mix its input with the uniform distribution, $(1-\lambda) f(X) + \lambda \mathbf{1}/k$ with $\lambda = 1/(n_{\mathrm{train}} + 1)$ and $n_{\mathrm{train}}$ the number of training points of the recalibration function.
The shift is $O(1/n)$, and calibration errors are always evaluated at the original predictions $f(X)$.

\begin{table}[ht!]
    \centering     
    \small
    \resizebox{\linewidth}{!}{\begin{tabular}{lrr @{\hskip 8mm} lrrr}
\toprule
\multicolumn{3}{c}{Binary} & \multicolumn{4}{c}{Multiclass} \\
\cmidrule(r{8mm}){1-3} \cmidrule(l){4-7}
Dataset & $n_{\mathrm{cal}}$ & $n$ & Dataset & $n_{\mathrm{cal}}$ & $n$ & $k$ \\
\midrule
\texttt{2dplanes} & 36,691 & 4,077 & \texttt{artificial-characters} & 9,196 & 1,022 & 10 \\
\texttt{adult} & 43,957 & 4,885 & \texttt{connect-4} & 60,801 & 6,756 & 3 \\
\texttt{ailerons} & 12,375 & 1,375 & \texttt{Diabetes130US} & 91,589 & 10,177 & 3 \\
\texttt{Amazon\_employee\_access} & 29,492 & 3,277 & \texttt{dilbert} & 9,000 & 1,000 & 5 \\
\texttt{APSFailure} & 68,400 & 7,600 & \texttt{eye\_movements} & 9,842 & 1,094 & 3 \\
\texttt{bank-marketing} & 40,689 & 4,522 & \texttt{GesturePhaseSegmentationProcessed} & 8,885 & 988 & 5 \\
\texttt{Click\_prediction\_small} & 35,953 & 3,995 & \texttt{har} & 9,269 & 1,030 & 6 \\
\texttt{delta\_elevators} & 8,565 & 952 & \texttt{Indian\_pines} & 8,229 & 915 & 8 \\
\texttt{eeg-eye-state} & 13,482 & 1,498 & \texttt{jungle\_chess\_2pcs\_raw\_endgame\_complete} & 40,337 & 4,482 & 3 \\
\texttt{electricity} & 40,780 & 4,532 & \texttt{kr-vs-k} & 25,250 & 2,806 & 18 \\
\texttt{fried} & 36,691 & 4,077 & \texttt{letter} & 18,000 & 2,000 & 26 \\
\texttt{houses} & 18,576 & 2,064 & \texttt{microaggregation2} & 18,000 & 2,000 & 5 \\
\texttt{jm1} & 9,796 & 1,089 & \texttt{okcupid-stem} & 45,710 & 5,079 & 3 \\
\texttt{kdd\_internet\_usage} & 9,097 & 1,011 & \texttt{pendigits} & 9,892 & 1,100 & 10 \\
\texttt{kick} & 65,684 & 7,299 & \texttt{robert} & 9,000 & 1,000 & 10 \\
\texttt{mammography} & 10,064 & 1,119 & \texttt{Traffic\_violations} & 63,306 & 7,034 & 3 \\
\texttt{mc1} & 8,519 & 947 & \texttt{volcanoes-b1} & 9,158 & 1,018 & 5 \\
\texttt{mozilla4} & 13,990 & 1,555 & \texttt{volcanoes-b2} & 9,601 & 1,067 & 5 \\
\texttt{nomao} & 31,018 & 3,447 & \texttt{volcanoes-b5} & 8,990 & 999 & 5 \\
\texttt{numerai28\_6} & 86,688 & 9,632 & \texttt{volcanoes-b6} & 9,117 & 1,013 & 5 \\
\texttt{PhishingWebsites} & 9,949 & 1,106 & \texttt{volkert} & 52,479 & 5,831 & 10 \\
\texttt{Run\_or\_walk\_information} & 79,729 & 8,859 &  &  &  &  \\
\bottomrule
\end{tabular}
}
    \label{table:datasets}
        \caption{\textbf{TabRepo datasets used in the experiments}, with the sizes of the calibration split ($n_{\mathrm{cal}}$) and of the test split ($n$), and the number of classes $k$ of the multiclass datasets. Each dataset contributes eight classifiers, one per model family.}
\end{table}

\subsection{Recalibration functions}
\label{app:recalibrators}

In \Cref{sec:exp:recalibrators}, every recalibration function is fitted on four of five folds of the test split (shuffled, with the same folds for all methods) and evaluated on the fifth, as in \Cref{alg:CE}.
All methods except WS-CatBoost are the implementations of the \texttt{probmetrics} package (version 1.3.0) released by \citet{berta2025rethinking}, with default hyperparameters.
\begin{itemize}[topsep=2pt,leftmargin=*]
    \item \textbf{Histogram binning} \citep{zadrozny2001obtaining} partitions the probability of class~1 into 10 bins of equal width (Uniform) or of equal mass on the training folds (Quantile), and predicts the frequency of class~1 in each bin; a bin left empty by the training folds predicts the overall frequency.
    \item \textbf{Isotonic regression} \citep{zadrozny2002transforming} fits a non-decreasing map from the probability of class~1 to the outcome (scikit-learn implementation).
    \item \textbf{Quadratic scaling} \citep{berta2025structured} is an unpenalized logistic regression on $(1, z, z^2)$, where $z = \log(f(X)/(1-f(X)))$ is the logit of the prediction.
    \item \textbf{Beta calibration} is the three-parameter logistic regression on $(\log f(X), -\log(1 - f(X)))$ with an intercept \citep{kull2017beta}.
    \item \textbf{Temperature scaling} \citep[TS,][]{guo2017calibration} rescales the log-probabilities by a single inverse temperature that minimizes the log-loss, found by bisection.
    \item \textbf{Structured vector and matrix scaling} \citep[SVS and SMS,][]{berta2025structured} first apply temperature scaling, then fit $\mathrm{softmax}((I + \mathrm{diag}(v)) z + b)$ (SVS) or $\mathrm{softmax}((I + \mathrm{diag}(v) + M) z + b)$ (SMS), with $M$ off-diagonal, on the temperature-scaled log-probabilities $z$. The penalties are ridge penalties of strength $k/n$ on $b$ and $v$ and $k(k-1)/n$ on $M$, and the fit uses BFGS.
\end{itemize}

\paragraph{WS-CatBoost.}
WS-CatBoost produces the out-of-fold predictions $\hat g(f(X_i))$ of \Cref{alg:CE} as follows. For each of the $m=5$ outer folds, with training part $\mathcal{T}$ and held-out part $\mathcal{E}$:
\begin{enumerate}[topsep=2pt,leftmargin=*]
    \item Fit temperature scaling on $\mathcal{T}$ (scikit-learn's implementation). Its log-odds (binary) or log-probabilities (multiclass) $s(u)$ are the warm start.
    \item Split $\mathcal{T}$ into eight inner folds $\mathcal{T}_1, \dots, \mathcal{T}_8$. For each $j$, fit a CatBoost classifier $h_j$ on $\mathcal{T}\setminus\mathcal{T}_j$ with features $u = f(X)$ (the probability of class~1 in the binary case, the probability vector otherwise), baseline $s(u)$, and the log-loss (binary) or multiclass cross-entropy. We use at most 100 trees, learning rate 0.2, early stopping after 20 rounds without improvement of the loss on $\mathcal{T}_j$, Bernoulli subsampling of half of the rows, 32 borders per feature, and CatBoost's defaults otherwise (in particular, trees of depth 6).
    \item Predict on $\mathcal{E}$ with the average $\hat g(u) = \frac{1}{8}\sum_{j=1}^8 \mathrm{softmax}(s(u) + h_j(u))$ (a sigmoid in the binary case).
\end{enumerate}
The $5 \times 8 = 40$ CatBoost models of a classifier are fitted in parallel. Classes that are absent from a training fold are added as rows of weight zero, so that every model outputs all $k$ classes.
A self-contained implementation is provided with the code.

\subsection{Semi-synthetic ground truth}
\label{app:groundtruth}

For each classifier $f$, we fit a TabICLv2 classifier \citep[package \texttt{tabicl} 2.0.3,][]{qu2026tabiclv2} on its calibration split. It has an internal ensemble of 8 members and default hyperparameters otherwise.
The context is a uniformly random subsample of 10,000 calibration points drawn without replacement, or the whole calibration split when it is smaller (all of them have at least 8,000 points).
The inputs are the predictions $f(X)$, reduced to the probability of class~1 in the binary case, and the targets are the observed labels.
The predictions of the fitted model on the test split define $C_i = \tilde g(f(X_i))$; classes that are absent from the context receive probability zero.
We use TabICL for three reasons. It is a strong nonparametric learner that needs no tuning. It can represent non-monotone calibration maps on the simplex of up to 26 classes. And it is not the recalibration function of any estimator in the benchmark, which prevents the ground truth from favoring one of them, as a CatBoost-based ground truth would favor V-ECE.
New labels are drawn independently as $Y_i \sim \mathrm{Categorical}(C_i)$, three times with different seeds, which also set the cross-validation folds of the estimators. Each draw is a separate semi-synthetic dataset: every estimator is run on each draw, its error is computed on each draw, and errors are averaged over the three draws (\Cref{app:metrics}).
The true calibration errors are
\begin{align*}
&\mathrm{CE}_{f,|\cdot|} = \frac{1}{n}\sum_{i=1}^n |f(X_i) - C_i| \, , \quad \mathrm{CE}_{f,(\cdot)^2} = \frac{1}{n}\sum_{i=1}^n (f(X_i) - C_i)^2 \quad \text{(binary, class-1 probabilities)} \, , \\
&\mathrm{CE}_{f,\|\cdot\|_p^q} = \frac{1}{n}\sum_{i=1}^n \|f(X_i) - C_i\|_p^q \quad \text{with } (p,q) \in \{(1,1), (2,1), (2,2)\} \quad \text{(multiclass)} \, .
\end{align*}
They do not depend on the label draws.
The map $\tilde g$ is itself an estimate and carries estimation noise. This noise is part of the ground truth, so it does not bias the comparison. It may however make the semi-synthetic calibration maps somewhat rougher than the real ones, and thus harder to recover for every estimator.

\subsection{Estimators}
\label{app:estimators}

\Cref{table:estimators} summarizes the estimators of \Cref{sec:exp:semisynthetic}. In the binary case all of them operate on the probability of class~1, and both $f(X)$ and $\hat g(f(X))$ below denote probabilities of class~1.

\begin{table}[ht!]
    \centering

    \small
    \setlength{\tabcolsep}{3.5pt}
    \resizebox{\linewidth}{!}{\begin{tabular}{llllcc}
    \toprule
    Estimator & Recalibration $\hat g$ / construction & Fitted & Functional & $k=2$ & $k>2$ \\
    \midrule
    Uniform-ECE & 15 bins, equal width & in-sample & binned distance & \checkmark & \\
    Quantile-ECE & 15 bins, equal mass & in-sample & binned distance & \checkmark & \\
    Debiased-ECE & 15 bins, equal width & in-sample & debiased binned, $(\cdot)^2$ only & \checkmark & \\
    Smooth-ECE & Gaussian smoothing & in-sample & smoothed, $|\cdot|$ only & \checkmark & \\
    Isotonic-ECE & isotonic regression & in-sample & plug-in distance & \checkmark & \\
    KDE-ECE & Beta / Dirichlet kernel & leave-one-out & plug-in distance & \checkmark & \checkmark \\
    \midrule
    Uniform-IS & 15 bins, equal width & in-sample & variational & \checkmark & \\
    Quantile-IS & 15 bins, equal mass & in-sample & variational & \checkmark & \\
    Isotonic-IS & isotonic regression (CORP) & in-sample & variational & \checkmark & \\
    \midrule
    V-ECE (Uniform) & 15 bins, equal width & 5-fold CV & variational & \checkmark & \\
    V-ECE (Quantile) & 15 bins, equal mass & 5-fold CV & variational & \checkmark & \\
    V-ECE (Isotonic) & isotonic regression & 5-fold CV & variational & \checkmark & \\
    V-ECE (KDE) & Beta / Dirichlet kernel & leave-one-out & variational & \checkmark & \checkmark \\
    V-ECE (TS) & temperature scaling & 5-fold CV & variational & \checkmark & \checkmark \\
    V-ECE (Quadratic) & quadratic scaling & 5-fold CV & variational & \checkmark & \\
    V-ECE (Beta) & Beta calibration & 5-fold CV & variational & \checkmark & \\
    V-ECE (SMS) & structured matrix scaling & 5-fold CV & variational & & \checkmark \\
    \textbf{V-ECE (Default)} & WS-CatBoost & 5-fold CV & variational & \checkmark & \checkmark \\
    \bottomrule
    \end{tabular}}
    \caption{\textbf{Estimators of the semi-synthetic benchmark}, grouped by family: plug-in estimators (ECE), variational estimators with an in-sample $\hat g$ (IS), and variational estimators with an out-of-sample $\hat g$, i.e.\ V-ECE with different recalibration functions. V-ECE (Default), in bold, is the configuration we recommend. \emph{Plug-in distance}: $\frac{1}{n}\sum_i d(f(X_i), \hat g(f(X_i)))$. \emph{Binned distance}: $\sum_b \frac{n_b}{n} d(\bar f_b, \bar y_b)$, comparing the average prediction $\bar f_b$ and the frequency $\bar y_b$ of class~1 in each bin. \emph{Variational}: the risk of $f$ minus the risk of $\hat g \circ f$, with the loss of \Cref{table:losses} that corresponds to the calibration error. The last two columns indicate whether each estimator is run in the binary ($k=2$) and multiclass ($k>2$) benchmarks: the binned, isotonic, smoothing, quadratic-scaling and Beta-calibration estimators are only defined for binary predictions, and SMS is only used in the multiclass case.}
    \label{table:estimators}
\end{table}

\paragraph{Binned estimators.}
The binned ECE partitions the probability of class~1 into 15 bins, either of equal width on $[0,1]$ (Uniform-ECE) or of equal mass, with edges at the quantiles of the test predictions (Quantile-ECE, duplicate edges merged). Writing $n_b$, $\bar f_b$ and $\bar y_b$ for the number of points, the average prediction and the frequency of class~1 in bin $b$, it returns $\sum_b \frac{n_b}{n} |\bar f_b - \bar y_b|$ for $\mathrm{CE}_{f,|\cdot|}$ and $\sum_b \frac{n_b}{n} (\bar f_b - \bar y_b)^2$ for $\mathrm{CE}_{f,(\cdot)^2}$ \citep{naeini2015obtaining, guo2017calibration}.
The debiased ECE \citep{kumar2019verified} subtracts from the squared version an unbiased estimate of the variance of each bin frequency, $\sum_b \frac{n_b}{n} \big[(\bar f_b - \bar y_b)^2 - \frac{\bar y_b (1 - \bar y_b)}{n_b - 1}\big]$, on the equal-width bins. As in the reference implementation, bins with a single point contribute zero.

\paragraph{Smooth ECE.}
Following \citet{blasiok2024smooth}, we place the residuals $Y_i - f(X_i)$ on a uniform grid of 1,024 points of $[0,1]$ at the position of their prediction, and smooth them with a Gaussian kernel of bandwidth $\sigma$ reflected at the boundaries. We then integrate the absolute value of the result, $\mathrm{smECE}_\sigma = \frac{1}{n}\int_0^1 \big|\sum_i K_\sigma(t - f(X_i))(Y_i - f(X_i))\big| \, dt$.
The reported value is taken at the self-consistent bandwidth $\sigma = \mathrm{smECE}_\sigma$, found by 40 steps of bisection on a logarithmic scale.

\paragraph{Kernel estimators.}
KDE-ECE \citep{popordanoska2022consistent} estimates $C$ at each test point by the leave-one-out Nadaraya--Watson regression \citep{nadaraya1964estimating},
\[
\hat g(f(X_i)) = \frac{\sum_{j \neq i} K_h(f(X_i); f(X_j))\, Y_j}{\sum_{j\neq i} K_h(f(X_i); f(X_j))} \, .
\]
The kernel is the Beta kernel $K_h(z; u) = \mathrm{Beta}(z; u/h + 1, (1-u)/h + 1)$ in the binary case and the Dirichlet kernel $K_h(z; u) = \mathrm{Dir}(z; u/h + 1)$ otherwise.
The bandwidth $h$ maximizes the leave-one-out likelihood of the corresponding kernel density estimate over a grid of 15 values log-spaced in $[10^{-5}, 10^{-1}]$ and $\{0.2, 0.4, 0.6, 0.8, 1\}$, on a random subsample of 2,000 test predictions (the likelihood costs $O(n^2)$ per bandwidth). We use the kernels of the authors' implementation.
KDE-ECE returns the plug-in distance and V-ECE (KDE) the variational estimate. The latter is a lower bound in expectation, because $\hat g(f(X_i))$ does not depend on $Y_i$.

\paragraph{Variational estimators.}
For an $L_p$ calibration error, the variational estimate is
\[
-\frac{1}{n}\sum_i \ell_{f(X_i)}(\hat g(f(X_i)), Y_i) \,
\]
with the losses of \Cref{cor:lp}. For a squared calibration error, it is the difference of Brier scores $\frac{1}{n}\sum_i \big(\|f(X_i) - Y_i\|_2^2 - \|\hat g(f(X_i)) - Y_i\|_2^2\big)$, computed on the class-1 probability in the binary case.
The in-sample variants (-IS) fit $\hat g$ on the whole test split and evaluate it on the same points. The V-ECE variants fit it on four of five folds and evaluate it on the fifth (or, for V-ECE (KDE), by leave-one-out). Each estimate is clipped at 0 after aggregating the folds, separately on each draw of the labels. The same clipping applies to the debiased ECE, the only other estimator that can be negative.
The histogram estimators use 15 bins, V-ECE (TS), (Quadratic), (Beta) and (SMS) the implementations of \Cref{app:recalibrators}, and V-ECE (Default) the WS-CatBoost recalibrator.

\subsection{Clipping at zero}
\label{app:clipping}

Calibration errors are nonnegative, so we clip every estimate at $0$.
We clip the estimate aggregated over the~$m$ folds of \Cref{alg:CE}, not the estimates of the individual folds: averaging clipped fold estimates would add the negative parts of all folds back to the estimate, an upward bias that does not vanish as the number of folds grows.
Only the variational estimators and the debiased ECE of \Cref{sec:exp:semisynthetic} can return negative values, so the other estimators are unaffected.

\paragraph{What clipping changes.}
Write $X$ for the unclipped estimate aggregated over the folds, computed on one dataset, and $\mathrm{CE} \ge 0$ for the true calibration error. Since $\mathrm{CE} \geq 0$, clipping has two properties that hold on every dataset, not only on average.
\begin{enumerate}[topsep=2pt,leftmargin=*,label=(\roman*)]
    \item \emph{It never turns an underestimate into an overestimate}: $\max(0, X) > \mathrm{CE}$ if and only if $X > \mathrm{CE}$.
    \item \emph{It never increases the error}: $|\max(0, X) - \mathrm{CE}| \le |X - \mathrm{CE}|$, because clipping is the projection onto $[0, \infty)$, which contains $\mathrm{CE}$. It strictly reduces the error whenever it acts, that is whenever $X < 0$.
\end{enumerate}
What clipping does break is the lower bound in expectation of \Cref{prop:lowerbound}: from $\mathbb{E}[X] \le \mathrm{CE}$ it no longer follows that $\mathbb{E}[\max(0, X)] \le \mathrm{CE}$.
The violation is however at most of the order of the sampling noise of the estimator. Since $\max(0, X) \le \max(0, \mathbb{E}[X]) + |X - \mathbb{E}[X]|$ and $\max(0, \mathbb{E}[X]) \le \mathrm{CE}$,
\begin{equation*}
    \mathbb{E}[\max(0, X)] \;\le\; \mathrm{CE} + \mathbb{E}\big|X - \mathbb{E}[X]\big| \;\le\; \mathrm{CE} + \sqrt{\mathrm{Var}(X)} \, ,
\end{equation*}
where $\sqrt{\mathrm{Var}(X)}$ decreases as $n^{-1/2}$.
The overshoot is noticeable only when $X$ is often negative, that is when the true calibration error is comparable to the sampling noise and the data cannot distinguish the model from a calibrated one. When $\mathrm{CE} > 0$, the probability that $X < 0$, and with it the overshoot, vanishes as $n$ grows.

\paragraph{Why this is harmless in practice.}
The lower bound in expectation describes an average over hypothetical redraws of the labels, whereas a practitioner observes a single set of predictions and labels and obtains a single estimate.
For that estimate, the guarantee that matters is the absence of a systematic upward bias: unlike in-sample estimators, the variational estimator cannot overfit its way to a large value.
Clipping preserves this up to the noise level, and on the observed dataset it can only help. It acts only when the raw estimate is negative, in which case $0$ is both closer to the truth and the honest answer: the data show no detectable miscalibration.

\paragraph{Effect in our experiments.}
In our semi-synthetic benchmark, each of the three draws of the labels is a separate dataset, on which every estimator is run, clipped and scored, exactly as a practitioner's single estimate would be.
By (i), clipping therefore cannot create any overestimate: for every estimator and every calibration error, the share of overestimated draws is the same with and without clipping.
By (ii), clipping never increases the error of an estimator for every calibration error.
For V-ECE (Default), 83 of the 528 binary draws give a negative estimate of $\mathrm{CE}_{f,|\cdot|}$ and 189 of $\mathrm{CE}_{f,(\cdot)^2}$, and clipping decreases its NMAE from $34.5\%$ to $29.8\%$ and from $36.0\%$ to $29.1\%$, respectively. In the multiclass case, where calibration errors are larger relative to the noise, the effect is smaller, for instance from $27.2\%$ to $26.6\%$ for $\mathrm{CE}_{f,\|\cdot\|_1}$.
The largest gains are for weak recalibration functions on squared calibration errors, whose unclipped estimates are often strongly negative; for instance, the NMAE of V-ECE (Quantile) goes from $281\%$ to $54\%$ for $\mathrm{CE}_{f,(\cdot)^2}$.

\subsection{Evaluation metrics}
\label{app:metrics}

Let $\widehat{\mathrm{CE}}_{e,c,r}$ be the estimate of estimator $e$ on classifier $c$ and label draw $r \in \{1, 2, 3\}$ for a given calibration error, $\mathrm{CE}_c$ the truth, which does not depend on the draw, and $a_{e,c,r} = |\widehat{\mathrm{CE}}_{e,c,r} - \mathrm{CE}_c|$. Every pair $(c, r)$ is treated as one dataset that a user could observe.
\begin{itemize}[topsep=2pt,leftmargin=*]
    \item The \emph{normalized absolute error} is $\mathrm{NAE}_e = \frac{1}{3N}\sum_{c}\sum_{r} a_{e,c,r} / \max_{e'} a_{e',c,r}$, where the maximum runs over the estimators that define the calibration error and $N$ is the number of classifiers.
    \item The \emph{normalized mean absolute error} is $\mathrm{NMAE}_e = \sum_{c,r} a_{e,c,r} / (3\sum_c \mathrm{CE}_c)$. It equals the average relative error $a_{e,c,r}/\mathrm{CE}_c$ weighted by $\mathrm{CE}_c$, and does not depend on the other estimators.
    \item The \emph{ratio} is the median over pairs $(c, r)$ of $\widehat{\mathrm{CE}}_{e,c,r}/\mathrm{CE}_c$, i.e.\ the typical share of the calibration error reported. The \emph{overestimation rate} is the share of pairs with $\widehat{\mathrm{CE}}_{e,c,r} > \mathrm{CE}_c$.
    \item \emph{Significance}: two-sided Wilcoxon signed-rank test on the paired differences $\bar a_{e,c} - \bar a_{\text{V-ECE (Default)},c}$ over classifiers, where $\bar a_{e,c} = \frac{1}{3}\sum_r a_{e,c,r}$, at level 5\%. Classifiers trained on the same dataset are not independent, so the $p$-values should be read as indicative.
    \item \emph{Model ranking}: for each dataset and label draw, the eight models are ranked by their true and by their estimated calibration error. We report the Spearman correlation between the two rankings, averaged over the three draws and then over datasets (with its standard error over datasets); the share of the 28 pairs of models ordered as by the truth, a pair tied by the estimate counting as half; and the share of draws on which the model with the smallest estimate is the best-calibrated one, shared among tied models (chance level 12.5\%).
\end{itemize}

\subsection{Hardware and runtime}
\label{app:hardware}

All experiments ran on a laptop with a 10-core Apple M2 Pro CPU.
WS-CatBoost fits its 40 models in parallel on all physical cores, while the other methods are not explicitly parallelized.
Runtimes are reported per 1,000 test samples in \Cref{table:recalibrators,table:estimators:binary,table:estimators:multiclass}.

\section{Additional results}
\label{app:exp:results}

\subsection{Choosing the recalibration function}

\Cref{table:recalibrators} completes \Cref{fig:recalibrators} with the share of classifiers on which each recalibration function returns the largest estimate, and with runtimes.
In the binary case, histogram binning with equal-width bins returns the largest estimate of $\mathrm{CE}_{f,|\cdot|}$ slightly more often than WS-CatBoost (23\% against 19\% of the classifiers), but it also falls far behind on many others, hence its worse average rank.

\begin{table}[ht!]
    \centering
    \small
    \begin{tabular}{ll ccc ccc c}
\toprule
& & \multicolumn{3}{c}{Average rank $\downarrow$} & \multicolumn{3}{c}{Largest estimate (\%) $\uparrow$} & Time \\
\cmidrule(lr){3-5} \cmidrule(lr){6-8}
& Recalibrator & $L_1$ & $L_2$ & squared & $L_1$ & $L_2$ & squared & [ms/1k] \\
\midrule
\multirow{6}{*}{\rotatebox{90}{Binary}} & \textbf{WS-CatBoost} & \textbf{3.15} & -- & \textbf{2.82} & 19 & -- & \textbf{31} & 1634 \\
 & Beta & 3.29 & -- & 3.06 & 12 & -- & 21 & 3.1 \\
 & Quadratic & 3.32 & -- & 3.19 & 14 & -- & 13 & 3.5 \\
 & Isotonic & 3.52 & -- & 3.51 & 17 & -- & 18 & 5.5 \\
 & Uniform & 3.67 & -- & 4.02 & \textbf{23} & -- & 11 & 0.44 \\
 & Quantile & 4.06 & -- & 4.40 & 15 & -- & 6 & 0.62 \\
\midrule
\multirow{4}{*}{\rotatebox{90}{Multiclass}} & \textbf{WS-CatBoost} & \textbf{2.01} & \textbf{1.98} & \textbf{1.98} & \textbf{47} & \textbf{48} & \textbf{46} & 2564 \\
 & SMS & 2.42 & 2.37 & 2.50 & 24 & 21 & 23 & 106 \\
 & SVS & 2.74 & 2.82 & 2.61 & 13 & 12 & 11 & 33.0 \\
 & TS & 2.84 & 2.83 & 2.91 & 17 & 19 & 19 & 9.3 \\
\bottomrule
\end{tabular}

    \label{table:recalibrators}
    \caption{\textbf{Recalibration functions inside the variational estimator} (\Cref{sec:exp:recalibrators}), on 149 binary and 150 multiclass TabRepo classifiers with real labels. \emph{Average rank}: rank of the estimate among the recalibration functions (1 for the largest), averaged over classifiers. \emph{Largest estimate}: share of classifiers on which the recalibration function returns the largest estimate. $L_1$ stands for $\mathrm{CE}_{f,|\cdot|}$ (binary) and $\mathrm{CE}_{f,\|\cdot\|_1}$ (multiclass), $L_2$ for $\mathrm{CE}_{f,\|\cdot\|_2}$, and \emph{squared} for $\mathrm{CE}_{f,(\cdot)^2}$ (binary) and $\mathrm{CE}_{f,\|\cdot\|_2^2}$ (multiclass). \emph{Time}: wall-clock time of the 5-fold estimate in milliseconds per 1,000 samples. Best values in bold.}
\end{table}

\subsection{Accuracy on the semi-synthetic benchmark}

\Cref{table:estimators:binary,table:estimators:multiclass} give the values plotted in \Cref{fig:accuracy}, together with the NMAE, the median ratio $\widehat{\mathrm{CE}}/\mathrm{CE}$ and runtimes. \Cref{fig:ratios} shows the ratio $\widehat{\mathrm{CE}}/\mathrm{CE}$ of four estimators on every classifier and label draw, as a function of the true calibration error.
Three observations stand out.
First, the V-ECE variants (Default, Quadratic, SMS and TS) have ratios below 1 on most draws, as \Cref{prop:lowerbound} predicts, and approach 1 as the calibration error grows.
Second, Quantile-ECE and KDE-ECE overestimate most strongly on the best-calibrated classifiers, where the sampling noise of the bin frequencies or kernel averages is largest relative to the calibration error. The same effect makes ECE report a positive value for a perfectly calibrated model.
Third, for the squared calibration errors, several variational estimators have a median ratio of 0 (Quantile-IS and V-ECE with histogram binning, isotonic regression or kernel regression). Their estimate is clipped at 0 on most draws: their recalibration function, too noisy or too coarse, has a larger Brier score than $f$ itself on held-out points. This is the full recalibration error of \eqref{eq:proper:bias}, whereas the $L_p$ losses only pay for the direction of $\hat g - f$ (\Cref{thm:alignment}).

\begin{table}[ht!]
    \centering 
    \small
    \setlength{\tabcolsep}{4pt}
    \begin{tabular}{l cccc cccc r}
\toprule
& \multicolumn{4}{c}{$\mathrm{CE}_{|\cdot|}$} & \multicolumn{4}{c}{$\mathrm{CE}_{(\cdot)^2}$} & Time \\
\cmidrule(lr){2-5} \cmidrule(lr){6-9}
Estimator & NAE $\downarrow$ & NMAE $\downarrow$ & ratio & over & NAE $\downarrow$ & NMAE $\downarrow$ & ratio & over & [ms/1k] \\
\midrule
Debiased-ECE & -- & -- & -- & -- & \textbf{0.09}$^\dagger$ & 25 & 0.97 & 45 & 0.06 \\
V-ECE (Beta) & \textbf{0.20} & \textbf{22} & 0.84 & 24 & 0.12$^\dagger$ & 38 & 0.48 & 16 & 3.4 \\
V-ECE (Quadratic) & 0.20 & 22 & 0.82 & 19 & 0.12$^\dagger$ & 39 & 0.50 & 15 & 3.6 \\
\textbf{V-ECE (Default)} & 0.21 & 22 & 0.84 & 27 & 0.11 & \textbf{25} & 0.48 & 16 & 1619 \\
V-ECE (Isotonic) & 0.25 & 25 & 0.74 & 19 & 0.15 & 37 & 0.21 & 9 & 5.8 \\
Quantile-ECE & 0.24 & 28 & 1.21 & 75 & 0.16 & 38 & 1.39 & 66 & 0.10 \\
V-ECE (TS) & 0.28 & 32 & 0.71 & 20 & 0.16 & 49 & 0.25 & 10 & 6.7 \\
Uniform-IS & 0.22 & 28 & 1.20 & 72 & 0.22 & 67 & 2.25 & 88 & 0.15 \\
Isotonic-ECE & 0.29 & 34 & 1.40 & 93 & 0.21 & 62 & 2.12 & 95 & 1.2 \\
Quantile-IS & 0.29 & 29 & 0.93 & 41 & 0.21 & 46 & 0.00 & 22 & 0.20 \\
V-ECE (Uniform) & 0.32 & 32 & 0.59 & 8 & 0.20 & 44 & 0.00 & 2 & 0.51 \\
V-ECE (KDE) & 0.30 & 31 & 0.67 & 14 & 0.22 & 58 & 0.00 & 1 & 39.2 \\
Uniform-ECE & 0.28 & 34 & 1.33 & 88 & 0.26 & 77 & 2.54 & 94 & 0.06 \\
V-ECE (Quantile) & 0.41 & 39 & 0.48 & 6 & 0.23 & 53 & 0.00 & 1 & 0.72 \\
Smooth-ECE & 0.33 & 37 & 1.43 & 87 & -- & -- & -- & -- & 1.3 \\
Isotonic-IS & 0.66 & 76 & 1.97 & 98 & 0.55 & 135 & 4.31 & 99 & 1.3 \\
KDE-ECE & 0.78 & 88 & 2.20 & 97 & 0.84 & 289 & 7.05 & 98 & 42.6 \\
\bottomrule
\end{tabular}

    \label{table:estimators:binary}
    \caption{\textbf{Accuracy of the binary estimators} on 176 semi-synthetic classifiers. \emph{NAE}: normalized absolute error (\Cref{fig:accuracy}). \emph{NMAE}: $\sum|\widehat{\mathrm{CE}} - \mathrm{CE}|/\sum \mathrm{CE}$, in \%. \emph{Ratio}: median of $\widehat{\mathrm{CE}}/\mathrm{CE}$. \emph{Over}: share of draws with $\widehat{\mathrm{CE}} > \mathrm{CE}$, in \%. All statistics are computed on each of the three label draws of every classifier (\Cref{app:metrics}). \emph{Time}: milliseconds per 1,000 samples for one estimate.
    $\dagger$: A dagger marks the estimators whose absolute errors are not significantly different from those of V-ECE (Default) (Wilcoxon signed-rank test, 5\%); all other estimators are significantly less accurate than V-ECE (Default), except V-ECE (Beta), which is significantly more accurate for $\mathrm{CE}_{f,|\cdot|}$. Estimates are clipped at 0 after aggregating the folds, so a median ratio of 0 means that the estimate is clipped on most draws. Best values in bold.}
\end{table}

\begin{table}[ht!]
    \centering 
    \small
    \setlength{\tabcolsep}{3pt}
    \resizebox{\textwidth}{!}{\begin{tabular}{l cccc cccc cccc r}
\toprule
& \multicolumn{4}{c}{$\mathrm{CE}_{\|\cdot\|_1}$} & \multicolumn{4}{c}{$\mathrm{CE}_{\|\cdot\|_2}$} & \multicolumn{4}{c}{$\mathrm{CE}_{\|\cdot\|_2^2}$} & Time \\
\cmidrule(lr){2-5} \cmidrule(lr){6-9} \cmidrule(lr){10-13}
Estimator & NAE $\downarrow$ & NMAE $\downarrow$ & ratio & over & NAE $\downarrow$ & NMAE $\downarrow$ & ratio & over & NAE $\downarrow$ & NMAE $\downarrow$ & ratio & over & [ms/1k] \\
\midrule
\textbf{V-ECE (Default)} & \textbf{0.27} & \textbf{26} & 0.70 & 2 & \textbf{0.26} & \textbf{26} & 0.70 & 2 & \textbf{0.15} & \textbf{44} & 0.34 & 0 & 2473 \\
V-ECE (SMS) & 0.32 & 37 & 0.64 & 2 & 0.31 & 37 & 0.65 & 2 & 0.17 & 58 & 0.25 & 2 & 102 \\
V-ECE (TS) & 0.36 & 44 & 0.63 & 2 & 0.36 & 44 & 0.61 & 2 & 0.20 & 63 & 0.12 & 1 & 12.8 \\
V-ECE (KDE) & 0.42 & 44 & 0.46 & 1 & 0.40 & 43 & 0.47 & 1 & 0.29 & 88 & 0.00 & 0 & 46.1 \\
KDE-ECE & 0.91 & 111 & 2.31 & 98 & 0.92 & 116 & 2.37 & 98 & 0.96 & 294 & 7.58 & 98 & 43.9 \\
\bottomrule
\end{tabular}
}
    \label{table:estimators:multiclass}
    \caption{\textbf{Accuracy of the multiclass estimators} on 168 semi-synthetic classifiers, with the columns of \Cref{table:estimators:binary}. V-ECE (Default) is significantly more accurate than every other estimator for the three calibration errors (Wilcoxon signed-rank test, 5\%).}
\end{table}

\begin{figure}[ht!]
    \centering
    \includegraphics[width=\textwidth]{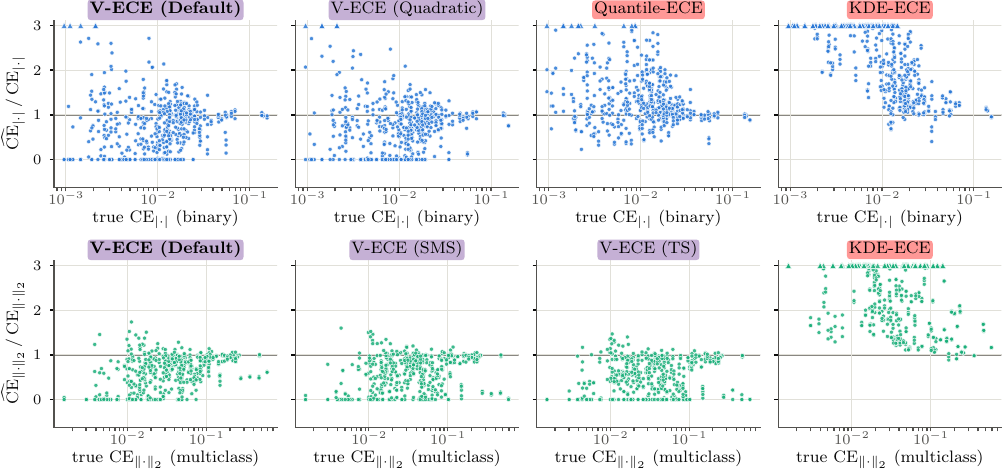}
    \caption{\textbf{Estimate divided by the true calibration error} on each semi-synthetic classifier and label draw, against the true calibration error (log scale), for $\mathrm{CE}_{f,|\cdot|}$ (binary, top) and $\mathrm{CE}_{f,\|\cdot\|_2}$ (multiclass, bottom). The horizontal line marks an exact estimate. Triangles are ratios outside the plotted range, drawn on its boundary.}
    \label{fig:ratios}
\end{figure}

\FloatBarrier

\subsection{Ranking models}
\label{app:ranking}

A calibration error is often used to compare models rather than for its value.
\Cref{fig:ranking} and \Cref{table:ranking:binary,table:ranking:multiclass} evaluate the estimators as tools for comparing models. Within each dataset, the eight models are ranked by their estimated and by their true calibration error.
In the multiclass case, V-ECE (Default) orders the models best for all three calibration errors, with Spearman correlations of $0.82$, $0.82$ and $0.79$ against at most $0.73$, $0.72$ and $0.64$ for the other estimators. With 21 datasets, the gap to the runner-up, V-ECE (SMS), is only significant for $\mathrm{CE}_{f,\|\cdot\|_2^2}$ ($p = 0.01$), and KDE-ECE identifies the best-calibrated model slightly more often for $\mathrm{CE}_{f,\|\cdot\|_1}$ and $\mathrm{CE}_{f,\|\cdot\|_2}$.
In the binary case, V-ECE (Default) is second for $\mathrm{CE}_{f,(\cdot)^2}$ ($0.63$, against $0.66$ for Isotonic-ECE; $p = 0.92$). For $\mathrm{CE}_{f,|\cdot|}$, the in-sample estimators order the models better ($0.73$ for Smooth-ECE and $0.72$ for Isotonic-ECE, against $0.62$; paired Wilcoxon test over the 22 datasets, $p = 0.003$ and $0.004$).
An overestimation shared by the models of a dataset does not change their ordering, which makes the bias of in-sample estimators less harmful for comparing models than for estimating the calibration error itself.

\begin{figure}[ht!]
    \centering
    \includegraphics[width=\textwidth]{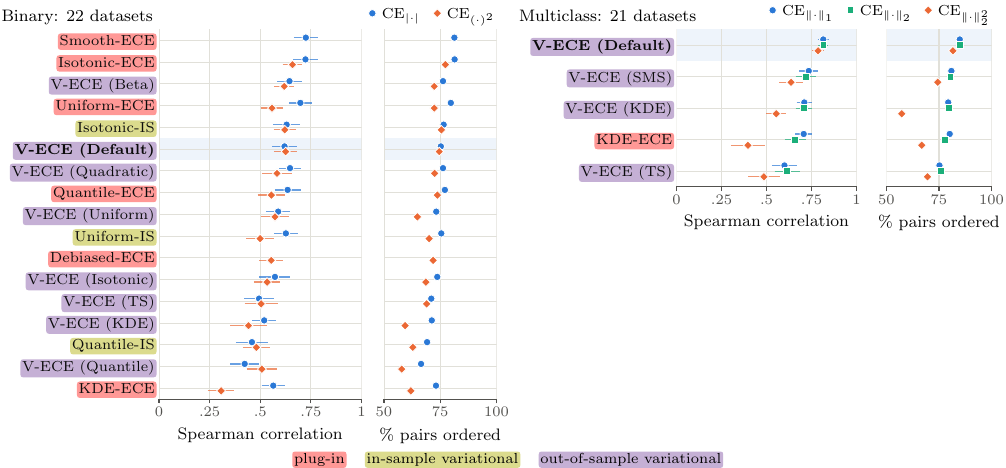}
    \caption{\textbf{Calibration error estimators as tools for comparing models.} Within each dataset, the eight models are ranked by their true and by their estimated calibration error. \emph{Spearman correlation}: rank correlation between the two orderings, computed on each label draw and averaged over draws and datasets (error bars: one standard error over datasets). \emph{\% pairs ordered}: share of the 28 pairs of models ordered as by the truth, a tied pair counting as half (50\% is chance).}
    \label{fig:ranking}
\end{figure}

\begin{table}[ht!]
    \centering 
    \small
    \begin{tabular}{l ccc ccc}
\toprule
& \multicolumn{3}{c}{$\mathrm{CE}_{|\cdot|}$} & \multicolumn{3}{c}{$\mathrm{CE}_{(\cdot)^2}$} \\
\cmidrule(lr){2-4} \cmidrule(lr){5-7}
Estimator & Spearman $\uparrow$ & pairs (\%) & top-1 (\%) & Spearman $\uparrow$ & pairs (\%) & top-1 (\%) \\
\midrule
Smooth-ECE & 0.78{\scriptsize$\pm$0.06} & 84 & 45 & -- & -- & -- \\
Isotonic-ECE & \textbf{0.79{\scriptsize$\pm$0.06}} & \textbf{85} & \textbf{59} & \textbf{0.72{\scriptsize$\pm$0.05}} & \textbf{80} & \textbf{50} \\
Isotonic-IS & 0.74{\scriptsize$\pm$0.07} & 81 & 27 & 0.70{\scriptsize$\pm$0.04} & 79 & 45 \\
Uniform-ECE & 0.79{\scriptsize$\pm$0.05} & 85 & \textbf{59} & 0.65{\scriptsize$\pm$0.05} & 77 & 32 \\
\textbf{V-ECE (Default)} & 0.66{\scriptsize$\pm$0.08} & 78 & 32 & 0.70{\scriptsize$\pm$0.06} & 79 & 38 \\
V-ECE (Beta) & 0.68{\scriptsize$\pm$0.07} & 79 & 32 & 0.66{\scriptsize$\pm$0.06} & 77 & 34 \\
Quantile-ECE & 0.71{\scriptsize$\pm$0.06} & 81 & 41 & 0.62{\scriptsize$\pm$0.08} & 77 & 36 \\
V-ECE (Quadratic) & 0.67{\scriptsize$\pm$0.07} & 79 & 34 & 0.64{\scriptsize$\pm$0.08} & 77 & 29 \\
V-ECE (Uniform) & 0.65{\scriptsize$\pm$0.07} & 77 & 18 & 0.64{\scriptsize$\pm$0.08} & 69 & 24 \\
Uniform-IS & 0.70{\scriptsize$\pm$0.06} & 79 & 41 & 0.58{\scriptsize$\pm$0.08} & 75 & 27 \\
Debiased-ECE & -- & -- & -- & 0.62{\scriptsize$\pm$0.07} & 76 & 32 \\
V-ECE (Isotonic) & 0.61{\scriptsize$\pm$0.10} & 76 & 41 & 0.60{\scriptsize$\pm$0.07} & 74 & 35 \\
V-ECE (Quantile) & 0.55{\scriptsize$\pm$0.08} & 73 & 40 & 0.56{\scriptsize$\pm$0.08} & 61 & 16 \\
V-ECE (KDE) & 0.64{\scriptsize$\pm$0.07} & 77 & 27 & 0.45{\scriptsize$\pm$0.10} & 62 & 19 \\
Quantile-IS & 0.51{\scriptsize$\pm$0.09} & 72 & 47 & 0.52{\scriptsize$\pm$0.08} & 66 & 24 \\
V-ECE (TS) & 0.54{\scriptsize$\pm$0.09} & 74 & 21 & 0.48{\scriptsize$\pm$0.09} & 70 & 22 \\
KDE-ECE & 0.59{\scriptsize$\pm$0.06} & 74 & 27 & 0.32{\scriptsize$\pm$0.08} & 63 & 23 \\
\bottomrule
\end{tabular}

    \label{table:ranking:binary}
    \caption{\textbf{Ordering of the eight models of each binary dataset} (22 datasets). \emph{Spearman}: rank correlation between the estimated and the true ordering, computed on each label draw and averaged over draws and datasets, $\pm$ one standard error over datasets. \emph{Pairs}: share of the 28 pairs of models ordered as by the truth, a tied pair counting as half. \emph{Top-1}: share of draws on which the model with the smallest estimate is the best-calibrated one, shared among tied models (chance: 12.5\%). Best values in bold.}
\end{table}

\begin{table}[ht!]
    \centering 
    \small
    \setlength{\tabcolsep}{4pt}
    \resizebox{\textwidth}{!}{\begin{tabular}{l ccc ccc ccc}
\toprule
& \multicolumn{3}{c}{$\mathrm{CE}_{\|\cdot\|_1}$} & \multicolumn{3}{c}{$\mathrm{CE}_{\|\cdot\|_2}$} & \multicolumn{3}{c}{$\mathrm{CE}_{\|\cdot\|_2^2}$} \\
\cmidrule(lr){2-4} \cmidrule(lr){5-7} \cmidrule(lr){8-10}
Estimator & Spearman $\uparrow$ & pairs (\%) & top-1 (\%) & Spearman $\uparrow$ & pairs (\%) & top-1 (\%) & Spearman $\uparrow$ & pairs (\%) & top-1 (\%) \\
\midrule
\textbf{V-ECE (Default)} & \textbf{0.89{\scriptsize$\pm$0.03}} & \textbf{89} & \textbf{57} & \textbf{0.88{\scriptsize$\pm$0.02}} & \textbf{89} & \textbf{52} & \textbf{0.83{\scriptsize$\pm$0.04}} & \textbf{86} & \textbf{60} \\
V-ECE (KDE) & 0.82{\scriptsize$\pm$0.03} & 86 & 52 & 0.82{\scriptsize$\pm$0.04} & 86 & 45 & 0.56{\scriptsize$\pm$0.08} & 59 & 14 \\
V-ECE (SMS) & 0.76{\scriptsize$\pm$0.06} & 84 & 43 & 0.76{\scriptsize$\pm$0.06} & 83 & 48 & 0.66{\scriptsize$\pm$0.07} & 78 & 36 \\
KDE-ECE & 0.73{\scriptsize$\pm$0.05} & 81 & 43 & 0.68{\scriptsize$\pm$0.06} & 79 & 43 & 0.41{\scriptsize$\pm$0.10} & 67 & 14 \\
V-ECE (TS) & 0.65{\scriptsize$\pm$0.07} & 77 & 24 & 0.64{\scriptsize$\pm$0.07} & 77 & 19 & 0.51{\scriptsize$\pm$0.08} & 71 & 21 \\
\bottomrule
\end{tabular}
}
    \caption{\textbf{Ordering of the eight models of each multiclass dataset} (21 datasets), with the columns of \Cref{table:ranking:binary}.}
    \label{table:ranking:multiclass}
\end{table}

\FloatBarrier

\section{Separating over- and under-confidence}
\label{app:synthetic}

We illustrate the decomposition of \Cref{sec:overunder} into over- and under-confidence on synthetic binary data, where the calibration function $g^\star$ is known in closed form.

\paragraph{Setup.}
We draw predictions $U \sim \mathrm{Beta}(0.5, 0.5)$ on $[0,1]$, playing the role of $f(X)$; this distribution places most predictions near $0$ and $1$.
Labels are then drawn as $Y \sim \mathcal{B}(g^\star(U))$, where $\mathcal{B}(u)$ is the Bernoulli distribution with parameter $u \in [0,1]$.
The true calibration error is $\mathrm{CE}_{f,|\cdot|} = \mathbb{E}[|U - g^\star(U)|]$, which we compute by Monte Carlo with $300{,}000$ samples, since $g^\star$ is known.
We consider three choices of $g^\star$, for which the model is over-confident, under-confident, or over-confident on some predictions and under-confident on others (\Cref{fig:overunder}, top).
All estimates are averaged over 10 independent runs.

\paragraph{Results.}
In the over-confident scenario, no under-confidence is detected and the over-confidence component accounts for the whole $\mathrm{CE}_{f,|\cdot|}$; the under-confident scenario is symmetric.
In the mixed scenario, the two components split the error evenly and still sum to $\mathrm{CE}_{f,|\cdot|}$, as predicted by \Cref{prop:overunder}.
In all three cases, the estimated total is close to the true $\mathrm{CE}_{f,|\cdot|}$, and slightly below it, as expected from the lower bound (\Cref{fig:overunder}, bottom).

\begin{figure}[ht!]
    \centering
    \subfigure{\includegraphics[width=0.25\textwidth]{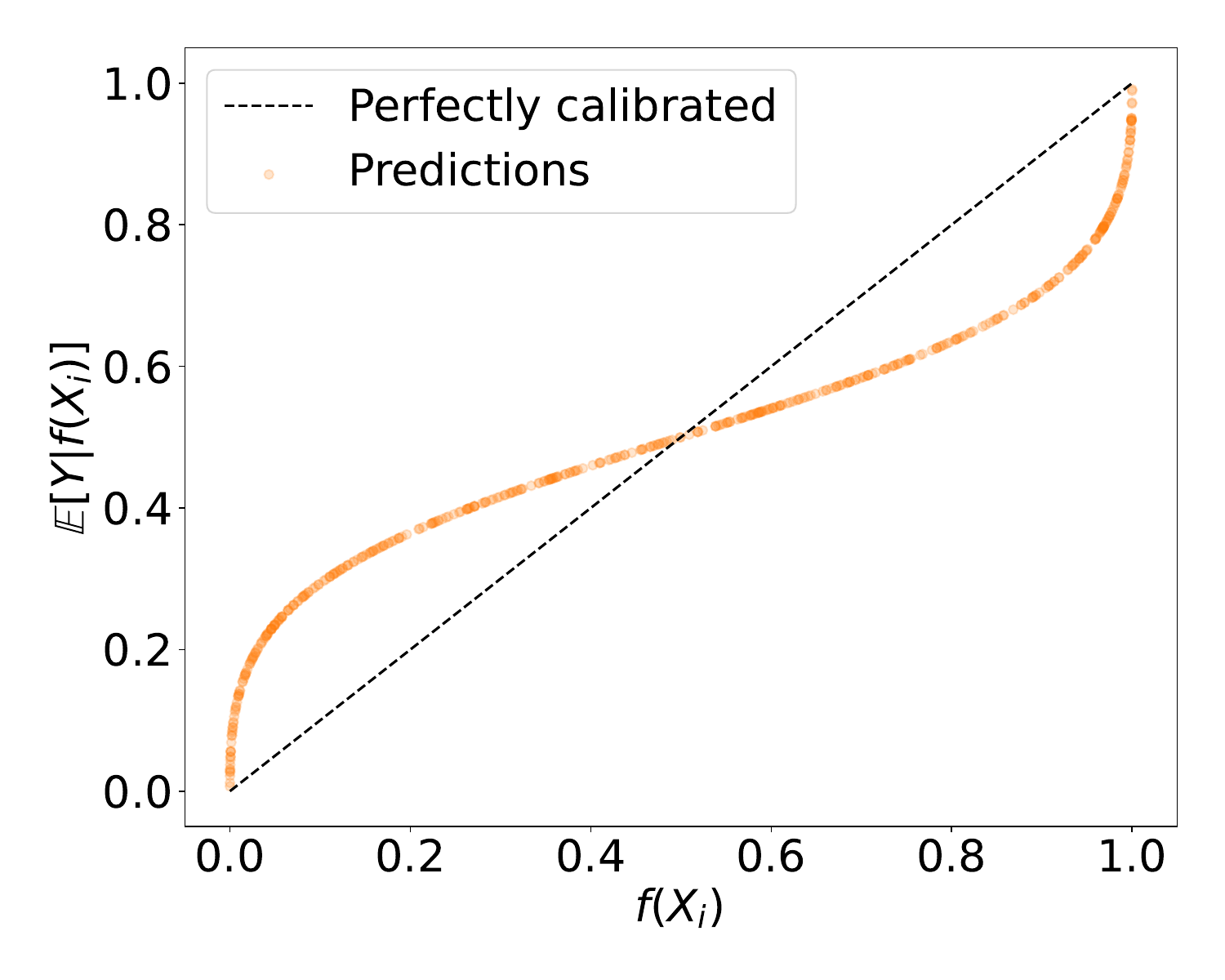}}\hspace{0.03\textwidth}
    \subfigure{\includegraphics[width=0.25\textwidth]{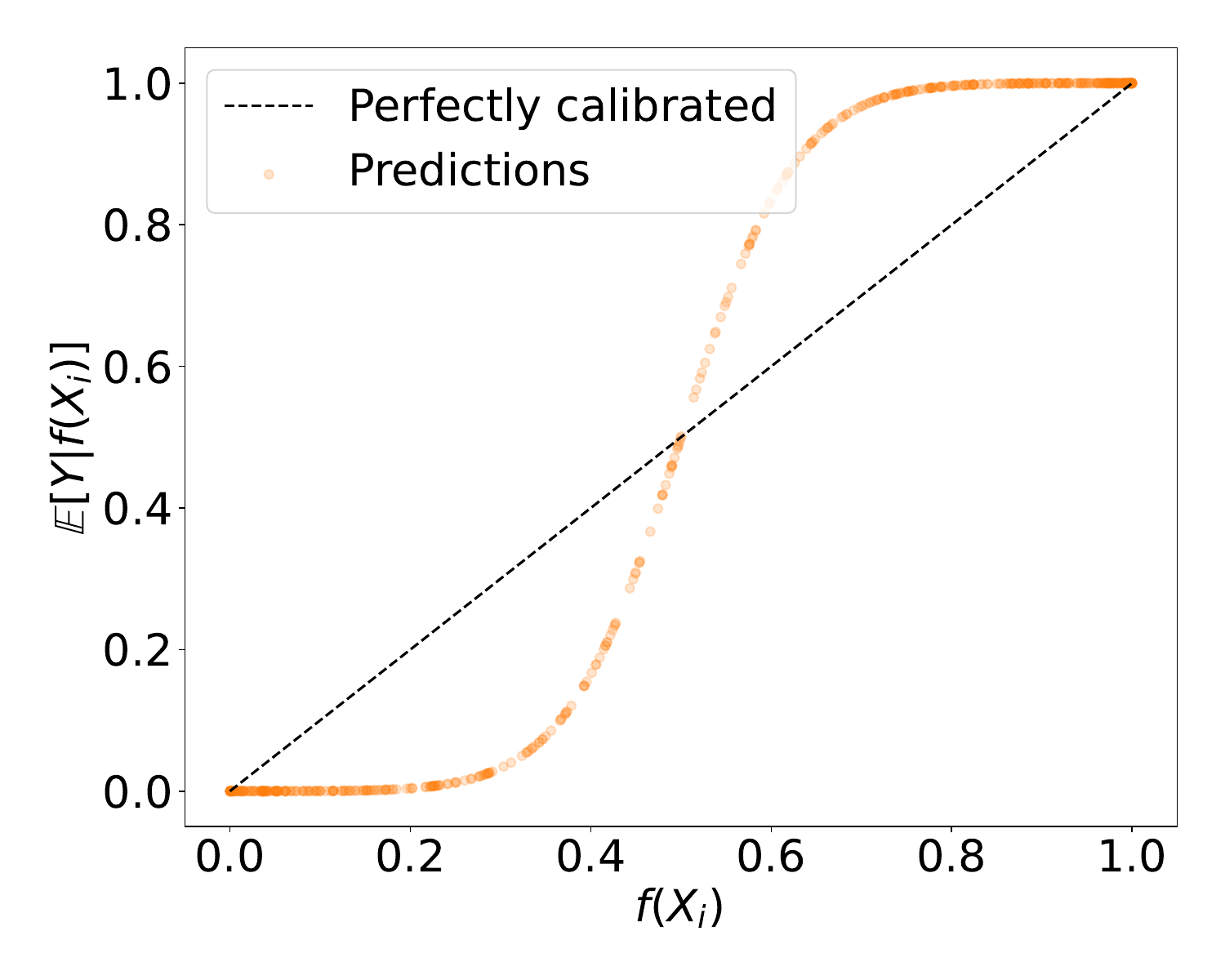}}\hspace{0.03\textwidth}
    \subfigure{\includegraphics[width=0.25\textwidth]{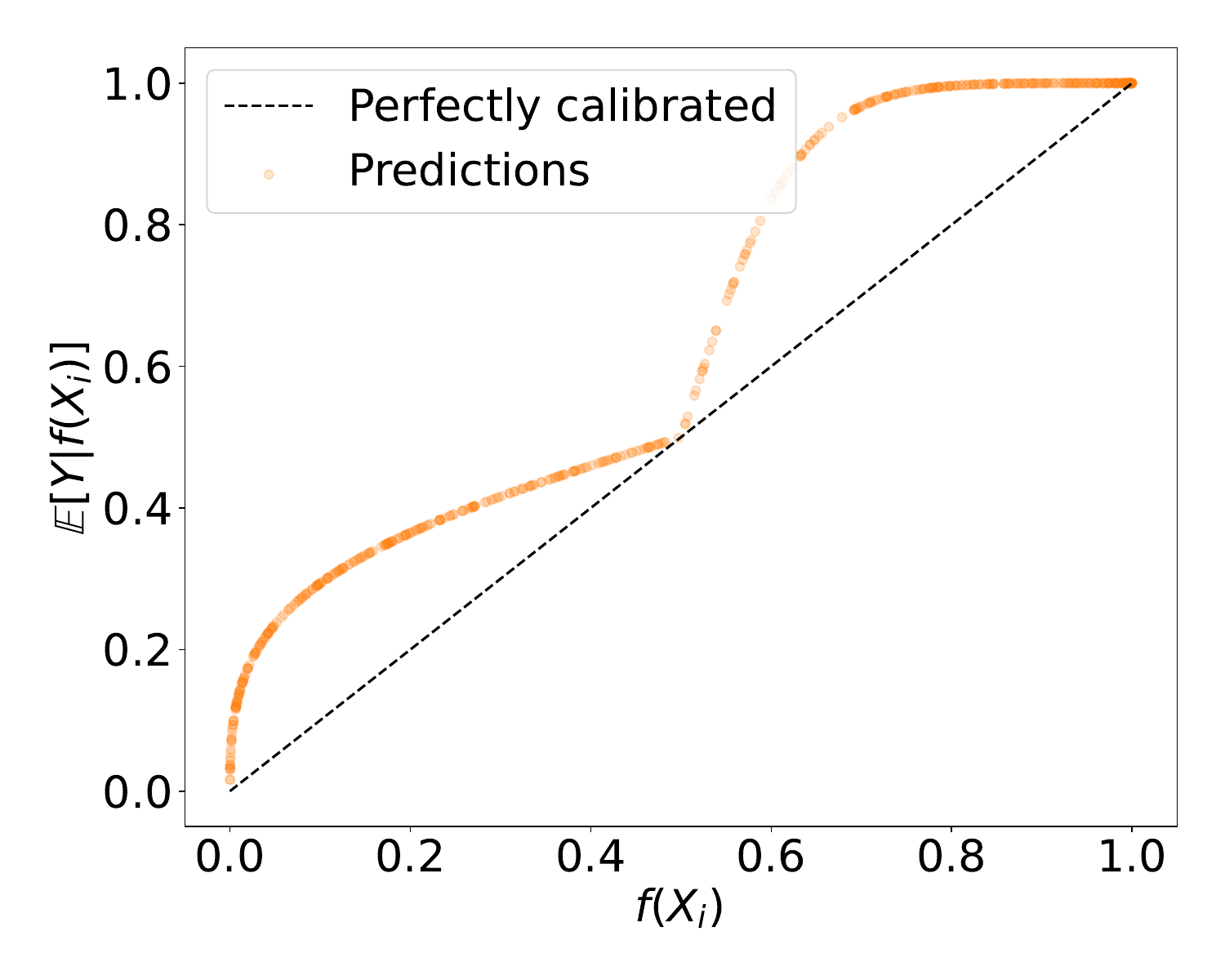}}

    \begin{tabular}{lccc}
    \toprule
    & Over-confident & Under-confident & Both\\
    \midrule
    $\widehat{\mathrm{CE}}^{+}_{f,|\cdot|}$ (over) & $0.128$ & $0.000$ & $0.062$ \\
    $\widehat{\mathrm{CE}}^{-}_{f,|\cdot|}$ (under) & $0.000$ & $0.126$ & $0.063$ \\
    $\widehat{\mathrm{CE}}_{f,|\cdot|}$ & $0.127$ & $0.125$ & $0.125$ \\
    $\mathrm{CE}_{f,|\cdot|}$ (true) & $0.129$ & $0.128$ & $0.128$ \\
    \bottomrule
    \end{tabular}
    \caption{\textbf{Separating over- and under-confidence.} Top: the three simulated scenarios, showing $C = \mathbb{E}[Y \mid f(X)]$ against the prediction $f(X)$ (points) and the diagonal of perfect calibration (dashed). Predictions are over-confident (left), under-confident (middle), or a mix of both (right). Bottom: the corresponding V-ECE estimates, clipped at $0$ and averaged over 10 independent runs.}
    \label{fig:overunder}
\end{figure}

\end{document}